\documentclass[twoside]{article}

\usepackage{amsmath}
\usepackage{amssymb}
\usepackage{chngcntr}
\usepackage{amsthm}
\usepackage{mathtools}
\usepackage{pifont}
\usepackage{dsfont}
\usepackage{bm}
\usepackage{mathrsfs}
\usepackage{eurosym}
\usepackage{tabularx}
\usepackage{booktabs}
\usepackage{cooltooltips}
\usepackage{svg}
\usepackage{colortbl}
\usepackage{color}
\usepackage{tikz}
\usepackage{subfigure}
\usepackage{wrapfig}
\usepackage{comment}
\usepackage{nicefrac}
\usepackage{subcaption}
\usepackage[font=small]{caption}

\usepackage[utf8]{inputenc} 
\usepackage[T1]{fontenc}    
\usepackage{hyperref}       
\usepackage{url}            
\usepackage{booktabs}       
\usepackage{amsfonts}       
\usepackage{nicefrac}       
\usepackage{microtype}      
\usepackage{xcolor}         
\usepackage{algorithm}
\usepackage{algorithmic}
\usepackage[capitalise]{cleveref}
\usepackage{enumitem}

\newcommand{\R}{\mathbb{R}}
\newcommand{\supp}{\operatorname{supp}}
\newcommand{\ball}{B}
\newcommand{\dimH}{\dim_{\mathrm H}}

\newcommand{\Hd}{\mathcal H}

\newtheorem{theorem}{Theorem}
\newtheorem{lemma}{Lemma}
\newtheorem{proposition}{Proposition}

\newtheorem{definition}{Definition}

\newtheorem{remark}{Remark}

\definecolor{mydarkblue}{rgb}{0,0.08,0.45}
\hypersetup{
  colorlinks=true,
  linkcolor=mydarkblue,
  citecolor=mydarkblue,
  urlcolor=mydarkblue
}

\usepackage[accepted]{aistats2026}
\usepackage[round]{natbib}
\begin{document}

%

%

\twocolumn[

\aistatstitle{Rethinking Intrinsic Dimension Estimation in Neural Representations}

\aistatsauthor{ Rickmer Schulte \And David R\"ugamer }

\aistatsaddress{ LMU Munich, MCML \And LMU Munich, MCML } ]

\begin{abstract}
  The analysis of neural representation has become an integral part of research aiming to better understand the inner workings of neural networks. While there are many different approaches to investigate neural representations, an important line of research has focused on doing so through the lens of intrinsic dimensions (IDs). Although this perspective has provided valuable insights and stimulated substantial follow-up research, important limitations of this approach have remained largely unaddressed. In this paper, we highlight a crucial discrepancy between theory and practice of IDs in neural representations, theoretically and empirically showing that common ID estimators are, in fact, not tracking the true underlying ID of the representation. We contrast this negative result with an investigation of the underlying factors that may drive commonly reported ID-related results on neural representation in the literature. Building on these insights, we offer a new perspective on ID estimation in neural representations.
\end{abstract}

\section{INTRODUCTION}
\label{sec:intro}

Intrinsic dimensions (IDs) play a central role in deep learning and have been the focus of research across a broad range of related studies. IDs are often encountered in the context of the so-called \textit{manifold hypothesis} \citep{tenenbaum2000global, fefferman2016testing}. The hypothesis postulates that many high-dimensional datasets frequently encountered in deep learning, such as image and text data, lie on or near a low-dimensional manifold despite being embedded in high-dimensional ambient spaces of dimension $d$, 
e.g., the number of pixels of an image. The hypothesis implies that a small number of dimensions $d_{\mathcal{M}}\ll d$ would theoretically suffice to fully characterize such datasets. This manifold dimension $d_{\mathcal{M}}$ is commonly referred to as the \textit{intrinsic dimension}.

The manifold hypothesis has been explored both empirically and theoretically in numerous studies. Although the validity of the hypothesis remains debated, many researchers attribute at least part of the success of deep learning to this phenomenon. In other words, the fact that deep learning models are able to \textit{learn} in the context of high-dimensional image and text data, and thereby escape the so-called \textit{curse of dimensionality} \citep{bellman1961, bishop2006pattern, bengio2013representation, goodfellow2016deep}, is said to be enabled by the presence of low IDs of the data that neural networks can adapt to \citep{chen2019efficient, schmidt2019deep, nakada2020adaptive, kohlerlanger2023, schulte2025adjustment}.

\paragraph{Related Literature} In recent years, there has been a rise in research investigating deep neural networks through the lens of IDs. Besides investigating IDs of frequently encountered datasets in deep learning \citep{pope2021intrinsic, konz2024intrinsic}, researchers have also aimed to understand the inner workings of these models by examining IDs of neural representations from different layers of the neural network and found consistently occurring patterns over various models \citep{gong2019intrinsic, ansuini2019intrinsic, cai2021isotropy, valeriani2023geometry, konz2024preprocessing, doimo2024representation, aljaafari2025trace, viswanathan2025geometry, cheng2025emergence}.

\begin{figure}
    \centering
    \includegraphics[width=0.98\columnwidth]{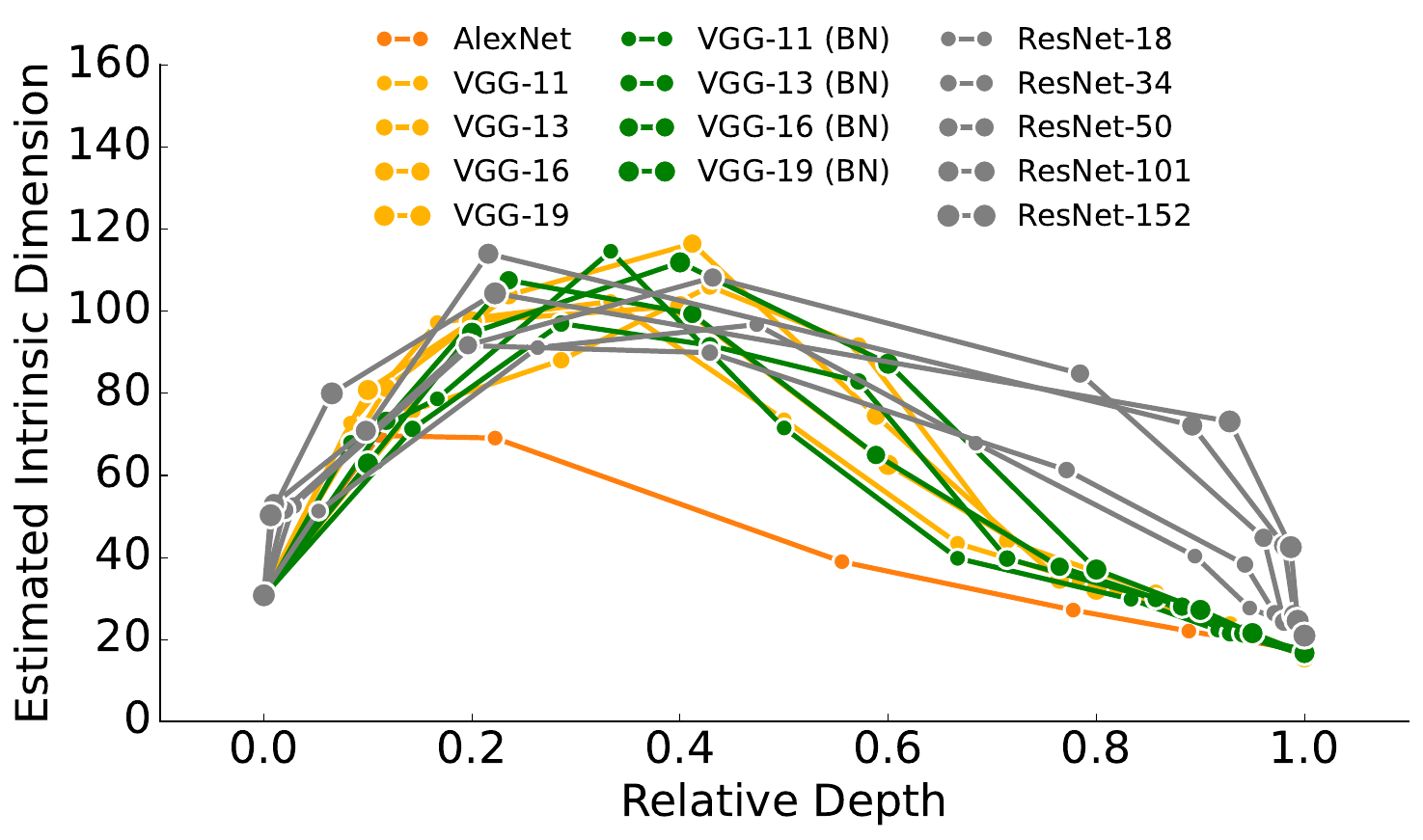}
    \caption{Layer-wise ID patterns for various architectures \citep[adapted from][]{ansuini2019intrinsic}. ID estimates are based on the TwoNN estimator. To compare models with varying number of layers, results are depicted over relative depth.}
    \label{fig:id_fig}
\end{figure}

\paragraph{Problem Statement} Investigating IDs of neural representations on various datasets and neural network architectures, ranging from vision to text-based models, all previous studies found ID estimates to vary over different network layers. Most strikingly, almost all of these studies find estimated IDs to increase in the early layers and to decrease in later layers (cf.~\cref{fig:id_fig}). Such patterns are then often interpreted as the emergence of abstractions or phase transitions \citep{cheng2025emergence}. 

While we do not question these previous empirical results, the question remains whether the observed phenomena can truly be attributed to the ID.

\paragraph{Our Contributions} 
In this work, we investigate this research question and provide the following contributions:

\begin{enumerate}[leftmargin=1.4em]
    \item We show that commonly used ID estimators are not only heavily biased in high dimensions (\cref{sec:bias_high_dim}), but are provably not tracking the true underlying IDs of layer-wise neural representations (\cref{sec:est_true_ids}).
    \item We further show that this result also holds in case the data does not lie on a single but a \textit{union of manifolds} (\cref{sec:union_manifolds}).
    \item Following this, we give a precise characterization of manifolds of LLM embeddings and hidden layer representations (\cref{sec:id_llms}). 
    \item Finally, we uncover driving forces behind layer-wise ID patterns and their connection to other metrics.
\end{enumerate}

\section{INTRINSIC DIMENSIONS \& ESTIMATORS}
\label{sec:main}

As described in the previous section, the core idea of the manifold hypothesis is that many deep learning datasets, represented in high-dimensional ambient spaces of dimension $d$, may lie on or close to a low-dimensional manifold with ID $d_{\mathcal{M}}\ll d$. 

Before diving into the analysis of IDs, we require a formal definition of IDs. While many different definitions are used in related literature, we will work with two of the most common IDs, the \emph{Hausdorff dimension} and the \emph{pointwise dimension}, and discuss corresponding estimators in the following. We provide a brief definition below, but refer the interested reader to the excellent surveys of \citet{camastra2016intrinsic} and \citet{binnie2025survey} that discuss several details, including other notions of IDs and estimators.

\paragraph{Intrinsic Dimensions} A common definition of ID is the \textit{Hausdorff dimension} \citep{hausdorff1918dimension}, which generalizes the concept of dimension to arbitrary sets in metric spaces. A formal definition can be found in \cref{app:def_haus_dim}. A very related notion of ID is the so-called \textit{pointwise dimension} \citep{young1982dimension}. Given that it is locally defined for each point (instead of all points as in the Hausdorff dimension), it is also known as the \textit{local Hausdorff dimension}. In case the lower and upper pointwise dimensions (formally defined in \cref{app:def_point_dim}) agree, the pointwise dimension defined at each point $x$ can be written as
\begin{equation}
    d_\mu(x)=\lim_{r\downarrow 0}\frac{\log \mu(\ball(x,r))}{\log r}\, ,
\end{equation}
where $\ball(x,r)$ corresponds to a ball with radius $r$ that is centered around the point $x$ and $\mu$ is a probability measure, such as Borel, or a measure supported on the set under study. In case $d_\mu(x)$ is $\mu$-a.s.\ the same for all points $x$, the pointwise dimension is said to be \textit{exact dimensional} and abbreviated by $d_\mu$. Further details can be found in \cref{app:def_point_dim}.

\paragraph{ID Estimators and its Targets} While there are generally many different ID estimators, we will focus on the ones that are most commonly applied in the analysis of neural representations, namely the Maximum Likelihood Estimator (MLE) \citep{levina2004maximum}, TwoNN \citep{facco2017estimating}, and variants thereof. 
Both ID estimators build on the ratios of nearest neighbor (NN) distances. For example, the MLE at point $x$ that considers $k$ neighbors, is defined as
\begin{equation}\label{eq:mle_est}
    \widehat d_{\mathrm{MLE}}(x)
\ =\ \left[\frac{1}{k-1}\sum_{i=1}^{k-1}\log\frac{T_k(x)}{T_i(x)}\right]^{-1},
\end{equation}
where $T_i(x)$ denotes the distance from $x$ to its $i$-th nearest neighbor. The final MLE estimate is obtained by averaging over the estimates of all points $x$, usually using the approach described in \citet{kay2005correction}. The TwoNN estimator is a special case of the MLE, using only the distances to the first two $k=2$ nearest neighbors (NNs). 
Although these estimators may differ in finite samples, it is fundamental for our analysis that both target the same underlying notion of ID, namely, the pointwise dimension. A formal derivation of this connection can be found in \cref{app:mle_2nn_targets}. 

\paragraph{Bias of ID Estimators in High Dimensions}\label{sec:bias_high_dim}

\begin{figure}
    \centering
    \includegraphics[width=0.98\columnwidth]{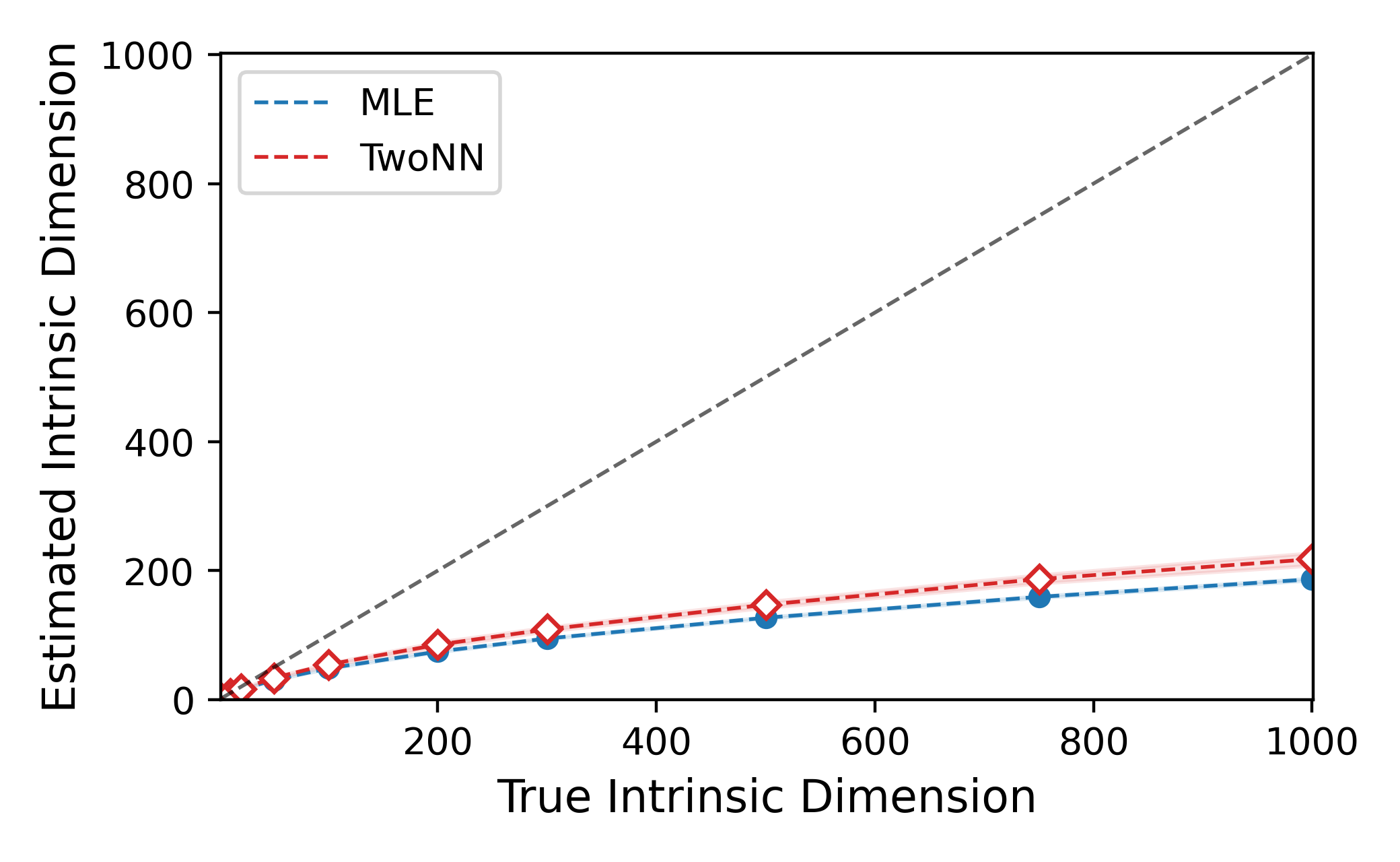}
    \caption{Estimated IDs using TwoNN and MLE vs.\ true ID of the manifold. Details are provided in \cref{app:add_exp}.}
    \label{fig:id_fig2}
\end{figure}

Similar to other ID estimators, the MLE and TwoNN estimators are known to be sensitive to their underlying assumptions and to underestimate the true ID in high dimensions. 
Although the phenomenon of underestimation is well recognized in the literature \citep{camastra2016intrinsic, binnie2025survey}, it is usually demonstrated for relatively small numbers of dimensions, revealing relatively small biases compared to ID estimates \citep{levina2004maximum}. However, we show that this bias grows drastically with increasing true ID (cf.~\cref{fig:id_fig2}). 
This is particularly concerning given that dimensions as those shown in \cref{fig:id_fig2} are highly relevant in modern deep learning. For example, the latest DINO embeddings increased from 1,535 in version 2 to 4,096 in version 3 \citep{simeoni2025dinov3}.

Given this underestimation, related papers usually state that ID estimates should be treated as a \textit{lower bound} of the actual ID \citep{ansuini2019intrinsic}. However, in order to provide meaningful insights into ID patterns of layer-wise representations of neural networks, the bias of such lower bound estimators must at least be consistent. If this is the case, \textit{relative comparisons} of estimates, e.g., by looking at the estimator's \emph{patterns}, would be rendered meaningful. However, as we will show in the next section, these estimated ID patterns also fail to track a (biased) version of the underlying ID and are therefore not at all indicative of it.

\section{INTRINSIC DIMENSIONS OF NEURAL REPRESENTATIONS
}

In the following, we show that ID estimators do not estimate the true IDs of neural representations and can also not be considered indicative of those. 
We derive this result for the \textit{pointwise dimension}. Analogous results for the \textit{Hausdorff dimension} are deferred to the Appendix. Before stating our results, we briefly introduce some relevant notation.

\paragraph{Notation}
Unless otherwise stated, the ambient spaces considered are Euclidean $(\R^d,\|\cdot\|)$. For \mbox{$f:\R^n\to\R^m$}, we say $f$ is \mbox{$L$-\emph{Lipschitz}} if $\|f(x)-f(y)\|\le L\|x-y\|$ for all $x,y \in \R^n$, and \emph{$(L,\alpha)$-Hölder}, $\alpha\in(0,1]$, if \mbox{$\|f(x)-f(y)\|\le L\|x-y\|^\alpha$}.
We denote the Hausdorff dimension by $\dimH$. 
For a probability measure $\mu$ on $\R^n$, the pushforward by $f$ is $\nu=f_\#\mu$, i.e.\ $\nu(\cdot)=\mu(f^{-1}(\cdot))$. For a feedforward neural network $(f_\ell)_{\ell=1}^L$ and an input law $\mu_0$, define $\mu_\ell:=(f_\ell)_\#\mu_{\ell-1}$.

\subsection{Intrinsic Dimensions of Layer-wise Neural Representations are Non-increasing}\label{sec:est_true_ids}

Our main finding rests on the observation that almost all neural network architectures are a composition of layer-wise Lipschitz mappings, 
and that common notions of IDs cannot increase under Lipschitz mappings.

\paragraph{Neural Networks are Lipschitz Mappings}
The observation that most neural networks are Lipschitz mappings follows from the fact that they are usually a composition of the following Lipschitz components:
\begin{itemize}[leftmargin=1.4em]
    \item Standard linear or convolutional layers \citep[e.g., proven in][Cor.~2.1]{kim2021lipschitz};
    \item Pointwise activations such as ReLU \citep[see, e.g.,][]{tsuzuku2018lipschitz} and Softmax \citep[see, e.g.,][Prop.~4]{gao2017properties};
    \item Pooling operators and residual additions \citep[derived in, e.g.,][]{tsuzuku2018lipschitz, bethune2022pay};
    \item Normalization layers such as BatchNorm and 
    \mbox{RMSNorm} \citep[e.g.,][]{tsuzuku2018lipschitz}. 
\end{itemize}
The conclusion that neural networks are Lipschitz mappings then just follows from the fact that compositions of Lipschitz mappings remain Lipschitz. We discuss the above results as well as the special case of self-attention in more detail in \cref{app:NN_Lipschitz}.

Using this insight, we now show in the following lemma that the pointwise dimension \textbf{cannot} increase under Lipschitz mappings.
\begin{lemma}[Pointwise dimensions under Lipschitz mappings]\label{lem:pointwise-holder}
Let \mbox{$f:\R^n\to\R^m$} be $L$\textup{-}Lipschitz, $\mu$ a Borel probability measure on $\R^n$, and $\nu=f_\#\mu$.
For $\mu$-a.e.\ $x$ with $y=f(x)$, we get for the upper $(\overline d_{\nu})$ and lower $(\underline d_{\nu})$ pointwise dimension that
$$\overline d_\nu(y)\le \overline d_\mu(x),
\qquad
\underline d_\nu(y)\le \underline d_\mu(x).$$
In particular, if the pointwise dimension $d_\nu$ exists \mbox{($\overline d_{\nu}=\underline d_{\nu}$)}, then 
$$d_\nu(y)\le d_\mu(x).$$
\end{lemma}
The corresponding proof can be found in \cref{app:proofs}.
A related result to Lemma \ref{lem:pointwise-holder}, studying the pointwise dimension under linear maps, can be found in {\citet[Lemma~4.5]{hochman2014lectures}}. Lemma \ref{lem:pointwise-holder} is more general in the sense that any linear map is Lipschitz, but not vice versa. In the following, we will omit a separate treatment of $\overline d_{\nu}$ and $\underline d_{\nu}$, assuming that $d_\mu$ exists.

Using the previous lemma, we can now show our first main result, namely the layer-wise monotonicity of the pointwise dimension. 

\begin{theorem}[Layer-wise monotonicity of pointwise dimensions]\label{thm:mono-pointwise} Let $f_1,\ldots,f_L$ be Lipschitz maps and set $\mu_\ell=(f_\ell)_\#\mu_{\ell-1}$.
Then for each $\ell \in \{1, \ldots, L\}$ and $\mu_{\ell-1}$-a.e.\ $x$ with $y=f_\ell(x)$, we have that
\[
d_{\mu_\ell}(y)\ \le\ d_{\mu_{\ell-1}}(x).
\]
\end{theorem}

\begin{remark}
    In case every $\mu_\ell$ is exact dimensional (i.e., $d_{\mu_\ell}(x)$ exists and equals a constant $d_\ell$ for $\mu_\ell$-a.e.\ $x$), then $d_\ell\le d_{\ell-1}$ for all $\ell$.
\end{remark}

\begin{proof}
Apply Lemma~\ref{lem:pointwise-holder} to each layer. 
In the exact-dimensional case, the pointwise dimensions are a.e.\ constant, so the a.e.\ inequality becomes $d_\ell\le d_{\ell-1}$.
\end{proof}

In other words, the theorem states that pointwise dimension cannot increase over the layers of any Lipschitz neural network.

In \cref{app:def_haus_dim}, we present an analogous result for the \textit{Hausdorff dimension} (\cref{lem:holder-hausdorff} and \cref{thm:mono-hausdorff}), namely that the Hausdorff dimension cannot increase under Lipschitz maps. These results have important implications for the ID estimation of neural representation as they uncover an important contradiction between theory and commonly found empirical ID-related results.

\begin{remark}
    A special case of Lemma \ref{lem:pointwise-holder} and \ref{lem:holder-hausdorff} (Appendix) arises for bi-Lipschitz mappings. In that case, the results in the two lemmas hold with equality. However, as most neural networks cannot be guaranteed to be compositions of bi-Lipschitz maps, this result might only be of minor relevance in our context. We therefore deferred the discussion to \cref{app:bilipschitz}. 
\end{remark}

\paragraph{A Contradiction} Theorem \ref{thm:mono-pointwise} and Theorem \ref{thm:mono-hausdorff} (Appendix) imply that the ID cannot increase over the layers of any Lipschitz neural network. This theoretical result stands in stark contrast to the increasing patterns of estimated IDs that are commonly observed across empirical studies investigating layer-wise ID in neural networks (e.g., ID patterns observed in \cref{fig:id_fig}). 
Given that the actual layer-wise ID cannot increase, the layer-wise ID patterns found by these studies cannot correspond to the true IDs of the neural representations, not even in a relative sense (the bias cannot be consistent). Hence, estimated neural ID patterns are not only strongly biased but also \textit{not at all indicative} of the underlying IDs of neural representations. 

\subsection{Extensions to Union of Manifolds}\label{sec:union_manifolds}

So far, we have introduced the classic version of the manifold hypothesis, in which data is assumed to lie on a single low-dimensional manifold. However, the hypothesis can be generalized to consider data that lies on a union of disconnected manifolds. A schematic visualization of the single and union of manifolds hypotheses can be found in \cref{fig:union_of_manifolds}. 

\subsubsection{Single vs.\ Union of Manifolds Hypothesis} 

While the idea of a union of manifold hypothesis has been prominent in the clustering literature for some time \citep{vidal2011subspace, elhamifar2011sparse}, \citet{brown2023union} more recently provided empirical evidence that many commonly used image datasets are more likely to live on a manifold union rather than a single manifold. Simplified, the key observation of their work is that images from the same classes (i.e., classes in the MNIST dataset) seem to share a common support, while images from different classes have disconnected supports. 
Moreover, as the union of manifold hypothesis allows images from different categories or classes to lie on disconnected manifolds, each of these class-specific manifolds may have a different ID. In line with this, \citet{brown2023union} find estimated IDs to vary between classes for several image datasets, providing further evidence for the union of manifold hypothesis.

\begin{figure}[t!]
    \centering
    \includegraphics[width=0.45\textwidth]{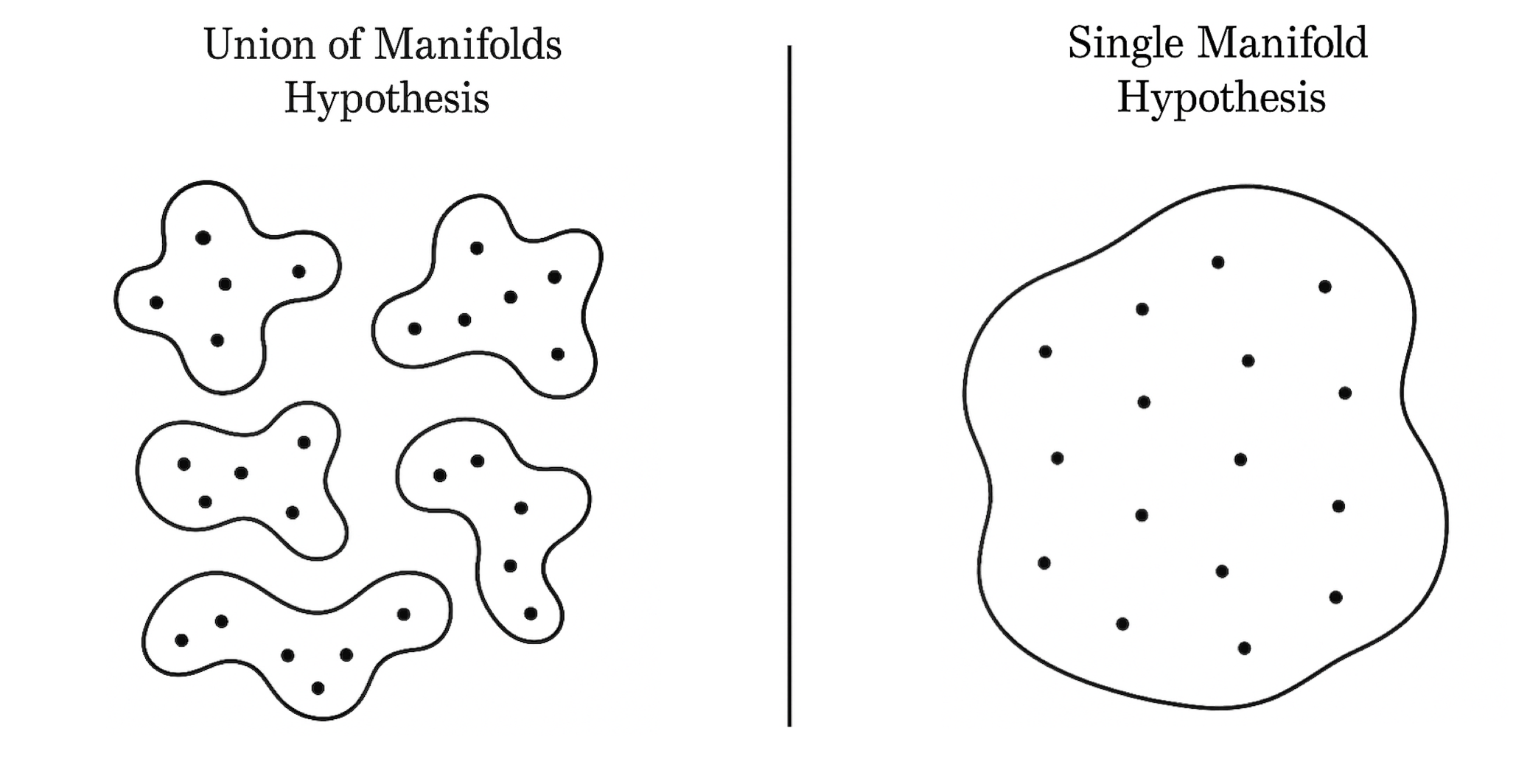}
    \vspace*{-2mm}
    \caption{Union of Manifolds vs.\ Single Manifold Hypothesis: Dots illustrate samples living on multiple disconnected manifolds (left) vs.\ a single manifold (right).}
    \label{fig:union_of_manifolds}
    \vspace{-1mm}
\end{figure}

\subsubsection{
IDs Under Unions of Manifolds}
Given the empirical evidence for the union of manifolds hypothesis (at least for image datasets), we investigate how our results from \cref{sec:est_true_ids} extend to this generalized version of the manifold hypothesis. 

\paragraph{Pointwise Dimension} Extending \cref{lem:pointwise-holder} to a union of manifolds is straightforward. Given that the dimension is only locally defined, the pointwise dimension is the same for points on the same manifold, but potentially different between points from disconnected manifolds. \cref{lem:pointwise-holder} can then be applied to each disconnected manifold of the union separately, yielding exactly the same conclusion as in \cref{lem:pointwise-holder} and \cref{thm:mono-pointwise}. That is, the pointwise dimension cannot increase under Lipschitz mappings, even if the data lies on a union of manifolds. 

\paragraph{Hausdorff Dimension} Similarly, the result of \cref{thm:mono-hausdorff} (Appendix) also extends to a union of manifolds, given that the considered set is a union of sets. However, as the Hausdorff dimension is defined globally for the entire set instead of locally for single points, the dimension can be made more precise in this case. For a (finite) union $Z = \bigcup_{i=1}^{p} M_i \subset \mathbb{R}^n$ of $p$ manifolds $M_i$, each with Hausdorff dimension $\dimH(M_i)=d_{\mathcal{M}_i}$, we get that the Hausdorff dimension of Z is $\dimH(Z)= \max_i d_{\mathcal{M}_i}$ {\citep[Prop.~2.12]{hochman2014lectures}}. Combining this with \cref{lem:holder-hausdorff} (Appendix), for a union of manifolds $Z$ under any Lipschitz map $\mbox{$f:\R^n\to\R^m$}$, we get that 
$$
    \dimH(f(Z))\le \dimH(Z)= \max_i d_{\mathcal{M}_i}.
$$
\paragraph{Class-specific ID Estimates} The previous two paragraphs show that the result from \cref{thm:mono-pointwise} and \cref{thm:mono-hausdorff} (Appendix) also hold under the union of manifold hypothesis. In addition, we can also verify that the mismatch between theory and practice of ID estimates remains. For this, we further refine the analysis of \cref{fig:id_fig} by estimating the class-specific IDs of layer-wise representations, depicted in \cref{fig:id_classes} for the ResNet-34 model. 
\begin{figure}[b!]
    \centering
    \includegraphics[width=0.49\textwidth]{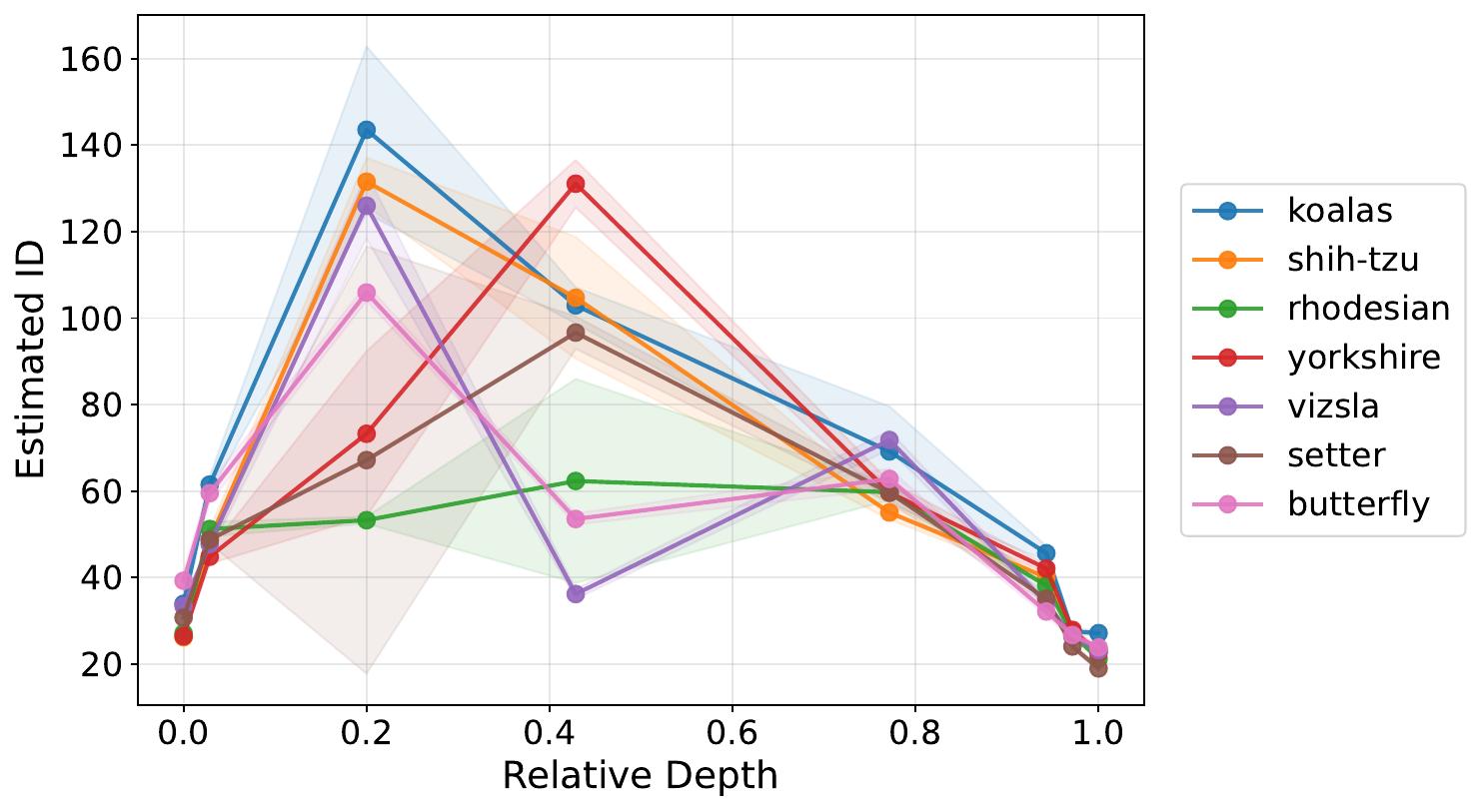}
    \caption{Estimated IDs of layer-wise representations from a ResNet-34 model for images with different categories (colors) from ImageNet. The x-axis shows the relative depth of model layers, and shaded areas show the estimated standard errors.
    }
    \label{fig:id_classes}
\end{figure}
That is, we now estimate IDs of layer-wise representations for different classes of the ImageNet dataset separately. 
Results show that even class-specific ID estimates increase over the layers of the model. 
Hence, also these class-specific ID estimates cannot be considered indicative of the underlying ID of class-specific representations. Similar results for other models can be found in \cref{app:IDs_per_class}.

\subsection{IDs \& Manifolds of LLM Representations}\label{sec:id_llms}

The previous sections discussed ID estimation of neural representations in general, covering various types of model architectures. While the experiments mainly deal with vision models, the theoretical results derived in \cref{sec:est_true_ids} also hold for other types of model architectures such as transformers. However, 
due to the increasing relevance and growing usage of large language models (LLMs) in recent years, we believe that providing a dedicated discussion and targeted results for LLMs is essential.

\paragraph{IDs in LLMs} The first question that arises in estimating ID in LLMs is what notion of ID one aims to estimate. This is relatively straightforward in the case of vision models, as each image gets mapped to a corresponding layer-wise representation, and IDs are then estimated layer-wise over the entire image dataset. However, in the case of LLMs, datasets at inference time are usually prompts, each corresponding to a single or multiple sentences. Each token in a sentence then gets mapped to a specific representation in the embedding layer of the LLM (cf.~\cref{fig:word_emd_llms}). In an autoregressive decoder, the hidden state of the current layer at a position $t$ is computed from the previous layer’s states at positions $\leq t$. The hidden state at the final position of a prefix thus encodes information about the entire input seen so far, giving it a special role in the prediction of the next layer and output.

\begin{figure}[htbp]
    \centering
    \includegraphics[width=0.78\columnwidth]{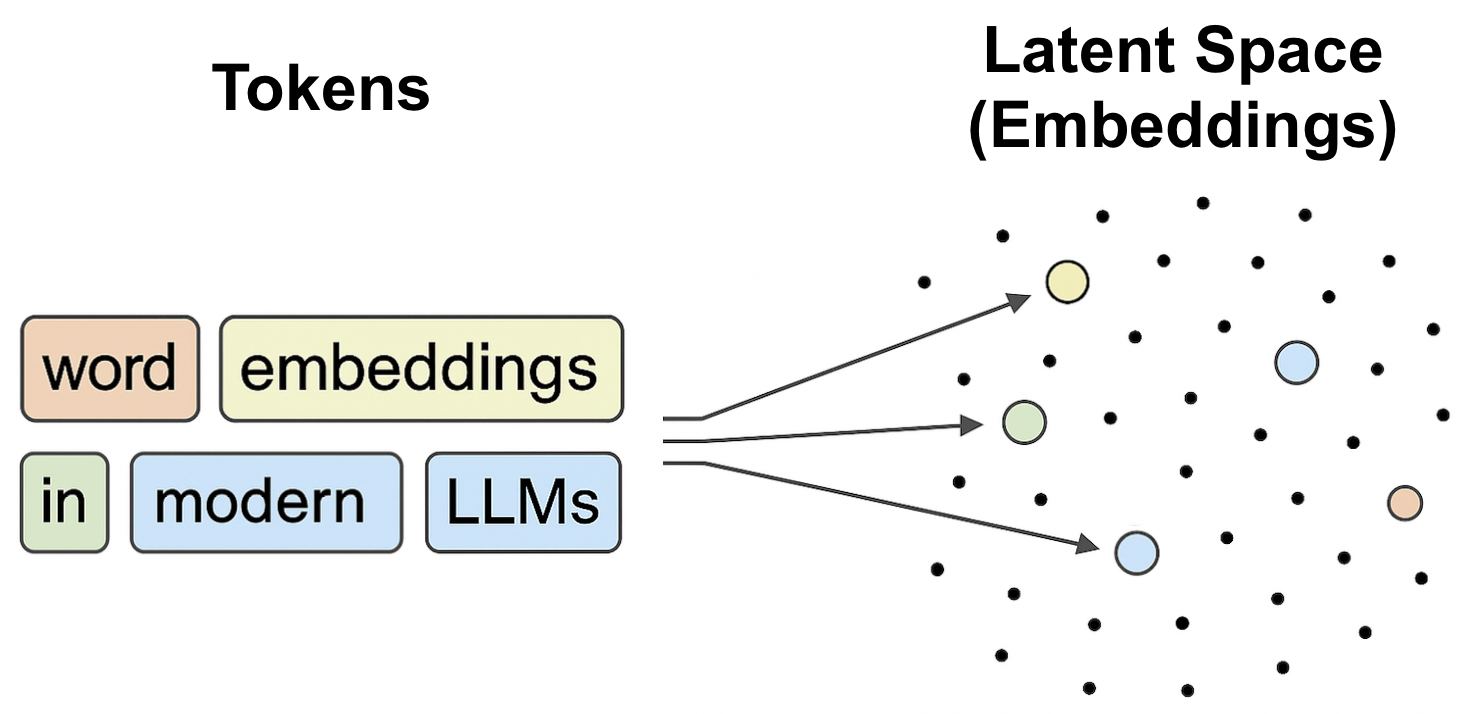}
    \caption{Visualization of word embeddings in LLMs: Each token in the prompt gets mapped to a unique point in the latent space of the embedding.
    }
    \label{fig:word_emd_llms}
\end{figure}

For this reason, the last hidden state can be used as a representative candidate in each layer for ID estimation \citep[see, e.g.,][]{cheng2025emergence}. 
Before discussing theoretical properties specific to LLMs, we will empirically analyze their ID patterns and relate findings to those of previous sections.

\subsubsection{ID Patterns of LLMs}\label{sec:id_patterns_llms}

For ID estimation in LLMs, we use 10k prompts from the popular wikitext \citep{merity2017pointer} dataset. For each prompt, we extract layer-wise representations from the pretrained Llama-3.1-8B \citep{grattafiori2024llama}, Mistral-7B-v0.3 \citep{jiang2023mistral7b}, and Pythia-6.9B \citep{biderman2023pythia}. Further details can be found in \cref{app:add_exp}.

\paragraph{Gride} ID estimation in LLMs is often based on the Gride ID estimator \citep{denti2022gride}. It can be seen as a variant of the TwoNN and MLE estimator. 
The only difference is that Gride uses distance ratios of non-consecutive NN pairs (e.g.\ $1^{st}$ vs.\ $2^{nd}$, $2^{nd}$ vs.\ $4^{th}$ 
etc.) rather than comparing only distances of the first nearest neighbors. Although the other two estimators could in principle also be used, Gride is sometimes preferred in the context of LLMs, as it can capture the local geometry beyond only the next neighbors, which can be advantageous in the high-dimensional representation spaces of LLMs. 
Important for our analysis is that Gride also targets the pointwise dimension analogous to TwoNN and MLE. Hence, the results from \cref{sec:est_true_ids} also apply.

\paragraph{ID Patterns in LLMs} The results of our LLM-based ID analysis are depicted in \cref{fig:ids_llms}. The figure shows the averages (thick line) of six different scalings of Gride (thin lines). Each scaling compares one of the $2^{nd}/1^{st}, \ldots, 64^{th}/32^{nd}$ NN distance ratios. Dotted lines correspond to the first scaling, which is equivalent to the TwoNN estimator. Further details are provided in \cref{app:add_exp}. Estimated ID patterns are relatively consistent among the considered LLMs and exhibit an increase, especially in the early layers. As before, this is in violation of the theoretical results about layer-wise IDs in neural networks, raising the same contradiction as in \cref{sec:est_true_ids}.

\begin{figure}[htbp]
    \centering
    \includegraphics[width=0.48\textwidth]{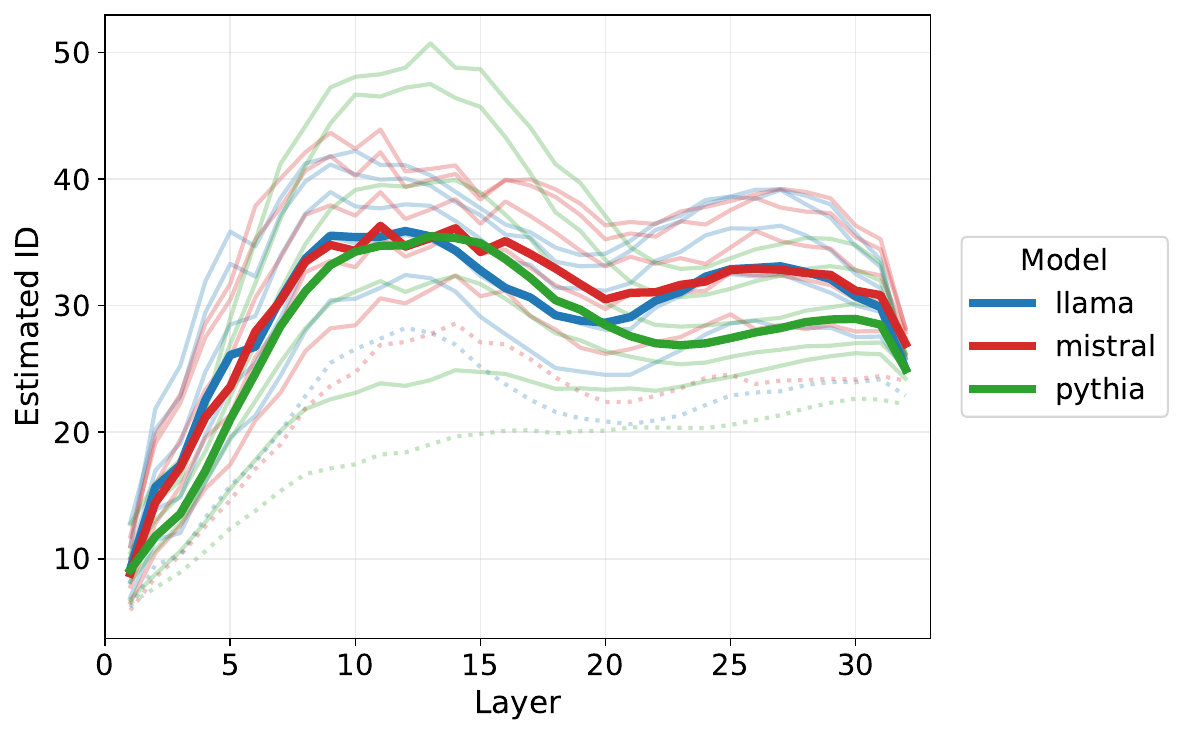}
    \caption{Estimated IDs of the layer-wise representations of various LLMs. Averaged ID estimates (thick lines) of six Gride scalings (thin lines), including the TwoNN (dotted).
    }
    \label{fig:ids_llms}
\end{figure}

\subsubsection{Exact IDs of LLMs representation}

As shown in \cref{fig:word_emd_llms}, LLMs 
map each token in a sequence to a unique embedding. While the number of unique points in this latent space depends on the vocabulary size of the model, they are always \textit{finite}. This has a simple but important implication. As ID definitions, such as the Hausdorff (or pointwise) dimension, are zero for point sets (or the distribution over those), the ID of such LLM embeddings must be zero. For finite sequences of tokens, this reasoning can be extended to hidden layers, given that they are continuous functions, which can map a finite point set only to a countable set of points. The latter has again a Hausdorff and pointwise dimension of zero. A formal result of this is given in \cref{lem:zero-dim-llm} (\cref{app:exact-id-llms}). We emphasize, however, that this result is specific to token-based inputs and therefore to LLMs. In \cref{app:diff-img-text}, we discuss why it does not transfer to image representations, and hence does not apply to vision models such as CNNs and vision transformers.

\begin{remark}
    While other concerns about the manifold hypothesis for LLM embeddings have been raised \citep{robinson2025token}, we believe that this simple yet insightful result above may be of interest on its own.
\end{remark}

\section{FORCES BEHIND LAYER-WISE ID PATTERNS}\label{sec:driving_forces}

In the previous section, we saw that increasing layer-wise ID patterns of neural representations are common to all discussed neural architectures. While the exact shape may differ between vision and language models (cf.~\cref{fig:id_fig} and~\cref{fig:ids_llms}), all exhibit increasing estimated IDs in the early layers. However, our theoretical investigation in \cref{sec:est_true_ids} showed that ID cannot increase over layers of such neural networks, implying that their estimates do not track the underlying true ID. 

Nonetheless, the consistency of ID patterns across pretrained models with different architectures suggests an important phenomenon that merits attention. We therefore investigate the forces underlying these layer-wise ID patterns, aiming to better understand what ID estimates may capture. While the experiments discussed here mainly focus on LLMs, analogous results for various vision models, including different CNNs and vision transformers (ViTs), are provided in \crefrange{app:id_amb_analysis}{app:entropy_vs_id}. 

\subsection{NN Distances}\label{sec:nn_dist}
Each of the considered ID estimators (MLE, TwoNN, Gride) is based on nearest-neighbor (NN) distances. 
Therefore, we investigate these distances in more detail in the following.

In \cref{fig:nn_dist_llms}, we plot the layer-wise NN distances that give rise to ID estimators in \cref{fig:ids_llms}. Results clearly show that NN distances are growing over the layers of all models. However, the $1^{st}$ and $2^{nd}$ NN distances grow by a similar amount as the $32^{nd}$ and $64^{th}$ NN distances (at least in the early layers). This implies that the last hidden state representations in each layer in LLMs move away from each other in the latent space, creating similar distances between all representations. 

This observation underlies the observed ID patterns. Considering again the construction of the ID estimators, each of them involves ratios of different pairs of NN distances. As the distances of all NNs grow by similar amounts (in early layers), their ratios shrink towards one. This is what drives estimated IDs to increase. In later layers, farther neighbors seem to grow slightly faster than closer NNs, which leads to a slight decrease in estimated IDs. 
\begin{remark}
    We omit the last layer from \cref{fig:nn_dist_llms} and most other plots in this section, since its representations are drastically altered by a LayerNorm transformation applied only at the final layer. As this effect is not central to our analysis, we exclude the last layer for readability and defer the full results to the Appendix.
\end{remark}\label{rm:last_layer}
The following sections explore different factors that could drive the increasing separation of representations in latent space.

\begin{figure}[t!]
    \centering
    \includegraphics[width=0.48\textwidth]{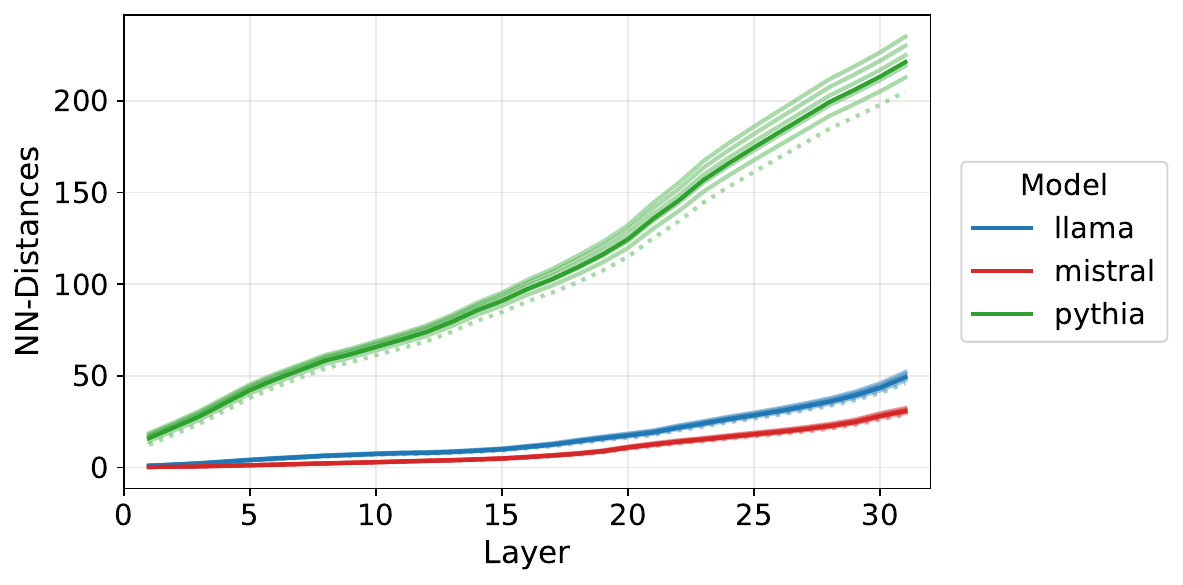}
    \caption{NN distances of layer-wise LLM representations (last layer excluded). For each model, each line (top to bottom) corresponds to the averages of $64^{th}$ \& $32^{nd}, \ldots, 2^{nd}$ \& $1^{st}$ NN distances. Solid lines denote the average over all 6 lines, with the TwoNN highlighted as a dotted line.
    }
    \label{fig:nn_dist_llms}
\end{figure}

\subsection{Ambient Space Dimension}\label{sec:amb_dim}
A natural candidate to explain the increasing separation of representations is the \textit{ambient space dimension}, i.e., the dimension of the representations in each layer.
This dimension may be a driving factor because, in high-dimensional spaces, points become increasingly separated, causing the distance to the nearest neighbor to approach that to the farthest neighbor \citep{beyer1999nearest, aggarwal2001surprising}.

Although the described mechanism seems plausible, we find strong evidence against the ambient space dimension being the key driving factor. The reasoning for the language models is simple. For all considered language models, the size of hidden representations remains constant over all layers, in our case 4096. Hence, the ambient space dimension cannot drive observed ID patterns. While the layer-wise ambient dimension can change for the considered vision models, \cref{fig:id_vs_emb_dim} in the Appendix also indicates that the layer-wise ambient space dimensions and ID-estimates seem rather unrelated.

\subsection{Cosine Similarity}\label{sec:cos_sim} 
Another important factor could be the \textit{cosine similarity}. The rationale behind this is that the last hidden states from different prompts should incorporate different contexts.  Therefore, their representations may become (nearly) orthogonal over the layers of the neural network. However, as depicted in \cref{fig:cos_sim}, the average cosine similarity between different last hidden state representations does not vary significantly over the model layers. While it is generally low (around 0.2) for the llama and mistral model, it is generally high (around 0.8) for the pythia model. Hence, cosine similarity does not seem to be a driving factor for increasing pairwise distances. 
\begin{remark}
    In contrast to the above results, \citet{viswanathan2025geometry} find the pairwise cosine similarity to increase over the layers of llama. However, they consider the cosine similarity between representations of the same prompt. In this case, an increasing cosine similarity is expected, given that the representations are updated jointly and impacted by the same tokens. 
\end{remark}
\begin{figure}[htbp]
    \centering
    \includegraphics[width=0.49\textwidth]{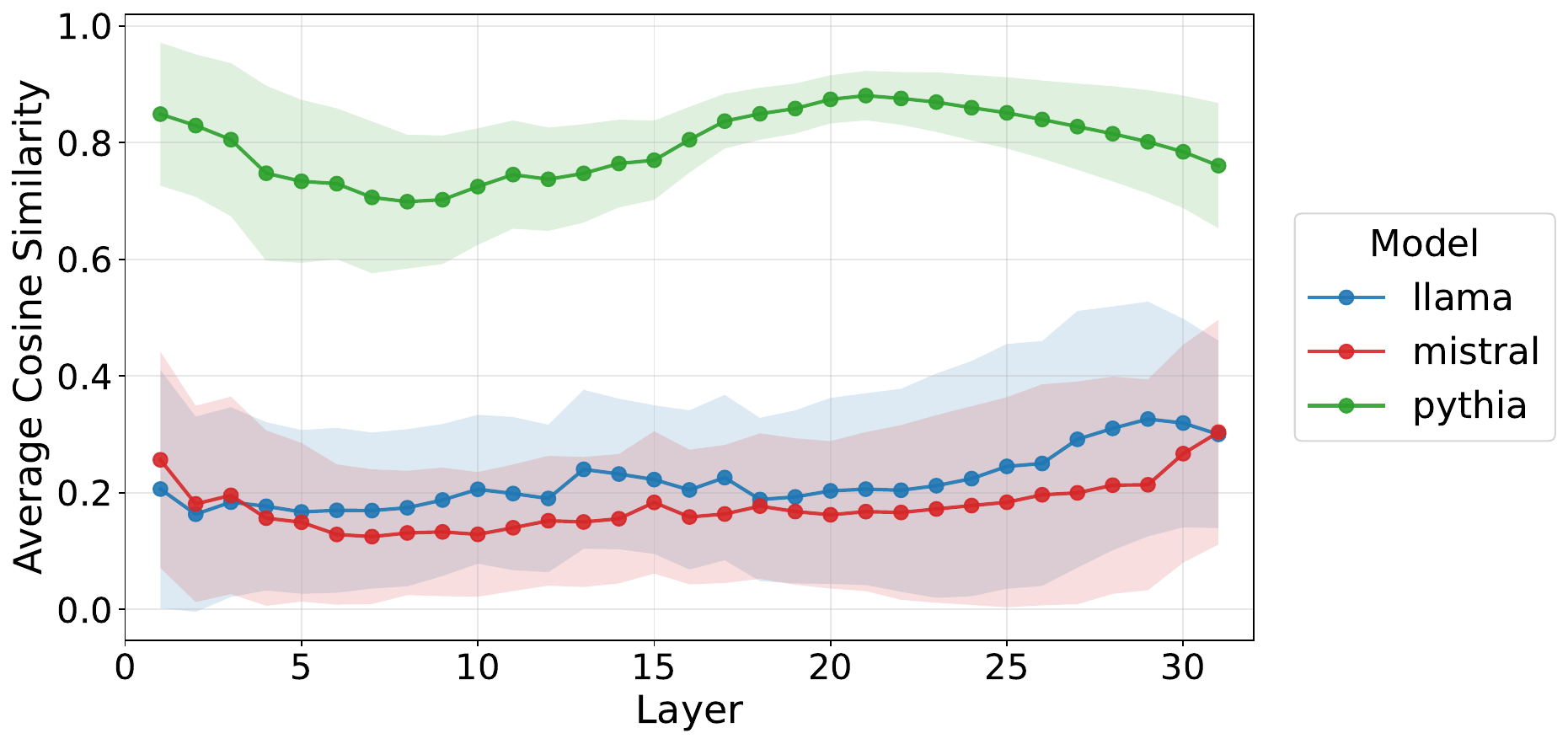}
    \caption{Average cosine similarity between layer-wise representations of various LLMs  (last layer excluded). The shaded area band represents twice the standard error. 
    }
    \label{fig:cos_sim}
\end{figure}

\subsection{Size of Representations}\label{sec:size_of_reps}

While the previous two candidates cannot sufficiently explain increasing NN distances, we consider the size of representations next. The size is naturally measured by the distance of each representation to the origin in latent space. As shown in \cref{fig:l2_dist_llms}, these distances grow across the hidden layers of the considered LLMs and exhibit patterns similar to the NN distances in \cref{fig:nn_dist_llms}. Similar layer-wise growth in the $L_2$ norm of hidden representations has also been observed across several other LLMs \citep{heimersheim2023residual, gupta2024geometric, lawson2025residual}. 
\begin{figure}[htbp]
    \centering
    \includegraphics[width=0.49\textwidth]{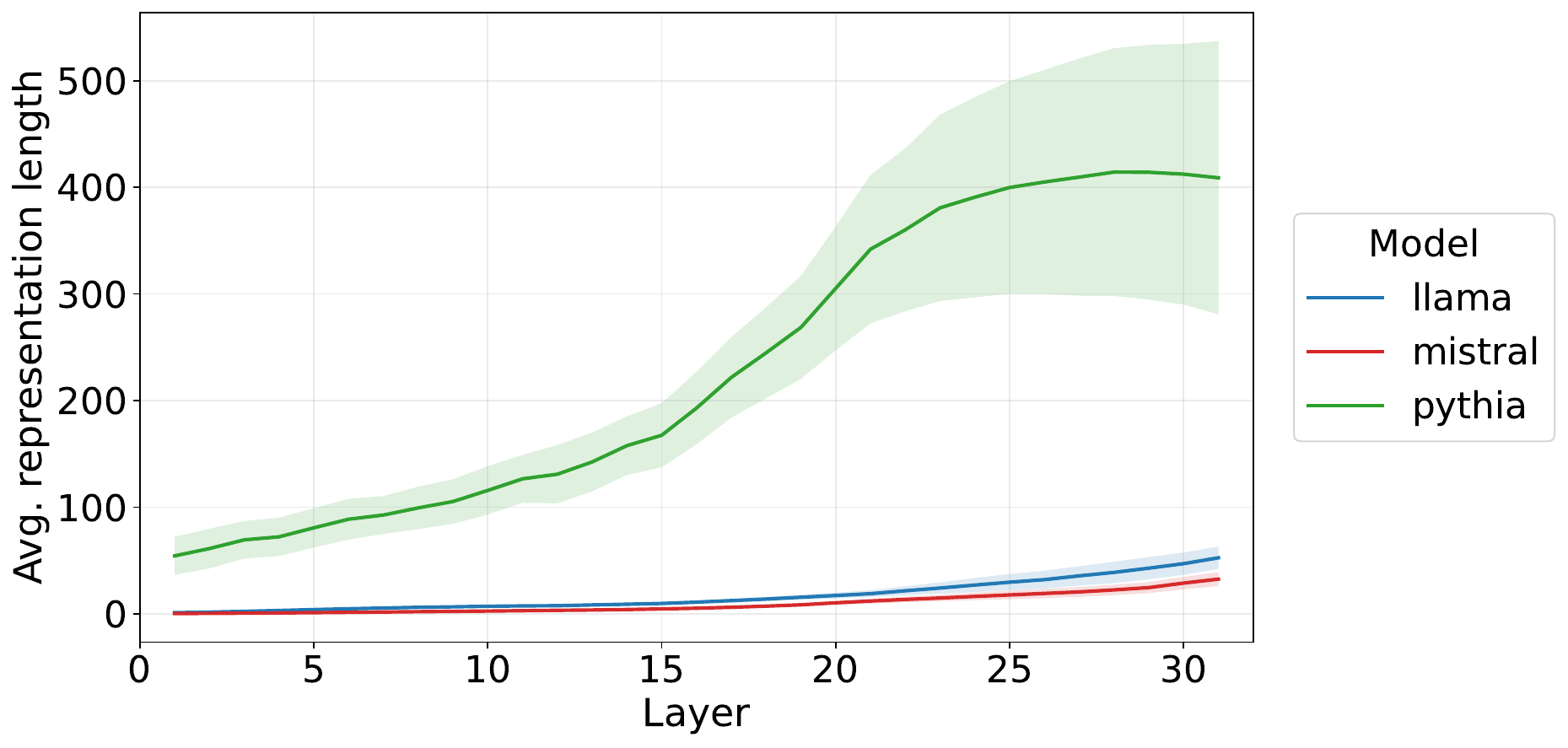}
    \caption{Average $L_2$ norm of layer-wise representations for llama, mistral and pythia (last layer excluded). The shaded area band represents twice the standard error. 
    }
    \label{fig:l2_dist_llms}
\end{figure}
\paragraph{Expansion in Latent Space} The layer-wise growth of the $L_2$ norms of
last hidden state representations corresponds to an expansion in latent space over the hidden layers. 
We believe that there is an intuitive explanation for this phenomenon: For an accurate next-token prediction, LLMs need to enrich last-token representations with the specific context of each prompt, thereby distinguishing it from other prompts. LLMs may achieve this over their hidden layers by sequentially moving the last hidden state representations into different areas of their latent space. This yields the observed expansion.

We have seen that ID estimates are affected by layer-wise expansions in latent space. However, it remains unclear whether this expansion occurs uniformly or primarily along certain directions. To shed light on this question, we investigate how representations are distributed in latent space and how this distribution changes across hidden layers, using entropy-based metrics in the next section.

\subsection{Entropy}\label{sec:other_metrics}
Apart from IDs, various other metrics have been used to analyze 
the layer-wise geometry of neural representations. Recently, \cite{skean2025layer} studied 
\textit{entropy}-based metrics
and also
found distinct layer-wise patterns across different architectures. While they consider different entropy-based metrics, all of them essentially measure the entropy of the distribution of eigenvalues of layer-wise representations. We provide a formal definition of these metrics in \cref{app:metrics}.

Inspired by these findings, we study layer-wise von Neumann entropy estimates and and find that they follow a pattern strikingly similar to that of ID estimates (cf.~\cref{fig:id_vs_entropy_llms}), rising strongly in early layers and falling again in later ones.
The strong connection between layer-wise ID and entropy patterns can also be found across various convolutional architectures and vision transformers (cf.~App.\ \cref{fig:id_vs_entropy_vits} and \cref{fig:id_vs_entropy}).

\begin{figure}[t]
    \centering
    \includegraphics[width=0.49\textwidth]{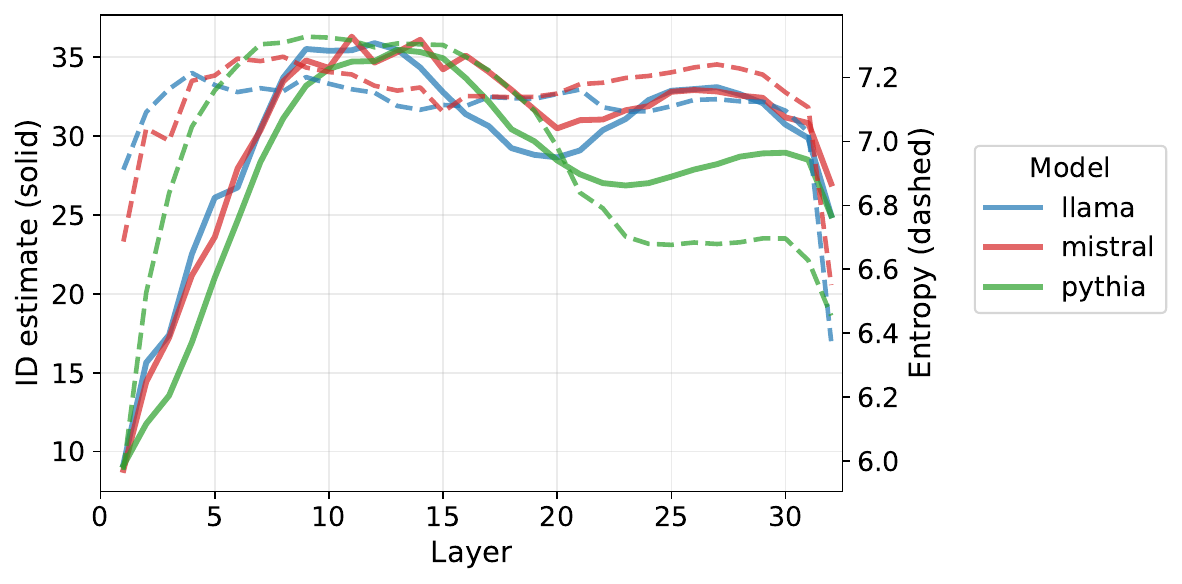}
    \caption{Estimated IDs (Gride) and entropy of layer-wise representation of various LLMs. Details in \cref{app:add_exp}.
    }
    \label{fig:id_vs_entropy_llms}
\end{figure}

\paragraph{Comparing IDs \& Entropy Estimates} Compared to ID estimators, layer-wise increases in entropy are well in line with theory and have a natural interpretation. As entropy estimates measure the spread of the eigenvalue distribution, they increase when representations are spread across many eigendirections and decrease when they are concentrated in a few. Hence, the patterns in \cref{fig:id_vs_entropy_llms} suggest that variance becomes more broadly distributed across linear directions in early layers and again more concentrated in later ones.

Small deviations between ID and entropy estimates might come from the information used. ID estimates only use local information from nearest neighbors, while entropy considers the entire distribution of all data points. Moreover, entropy captures the extent to which variance is distributed across linear directions in latent space. This not only gives estimators a clear interpretation, but is also in line with empirical evidence in modern LLMs, suggesting information is encoded linearly in the representations \citep{marks2024linear, jiang2024linear, park2024linear, park2025linear}. The strong connection between the two metrics in \cref{fig:id_vs_entropy_llms} may therefore indicate that the observed ID patterns are driven by layer-wise changes in the spread of variance
across \textit{linear} directions in the latent space.

\paragraph{Theoretical Connection} To the best of our knowledge, this is the first work to identify a striking resemblance between ID and entropy estimates for neural representations. While prior work has explored the theoretical connection between the two \citep{costa2006determining, bailey2021relationships, bailey2022local}, it focuses on entropy on the manifold, whereas our analysis concerns entropy in the ambient space.

\section{CONCLUSION AND FUTURE OUTLOOK}
\label{sec:conclusion}

\subsection{New Perspective \& Impact}

\paragraph{A New Perspective on ID Estimates of Neural Representations}\label{sec:new_perspective_ids}

Our theoretical results, in combination with our experiments, show that common ID estimates of neural representations do not recover the true IDs. However, the fact that the estimated layer-wise ID patterns are consistent across independently trained neural models with varying architectures (both for vision and text models) indicates that they capture an important geometric phenomenon inherent to neural representations more generally. Instead of interpreting these as IDs of layer-wise neural representation manifolds, our analysis suggests that it seems more appropriate to view them as reflecting a distinct geometric characteristic of layer-wise neural representations that 
can also be captured with 
entropy-based metrics. Namely, how variance is distributed across linear directions in latent space and how this distribution changes over the hidden layers.

\paragraph{Impact} Studying the IDs of layer-wise neural representations, especially in the context of transformer-based models, has become increasingly popular in recent years. 
Studies have considered layer-wise ID estimation of neural representations in a variety of settings and across a broad range of models, including \textit{decoder-only} models \citep{doimo2024representation, viswanathan2025geometry, cheng2025emergence}, \textit{encoder-decoder} models \citep{valeriani2023geometry}, and \textit{vision models} including different convolutional architectures and ViTs \citep{ansuini2019intrinsic,kvinge2023internal, wang2024exploring, konz2024intrinsic, shah2025topographic, wang2025textural, roschmann2025time}. Given that increasing ID patterns are commonly observed in these studies, our results suggest that such patterns should be interpreted with care. More broadly, they indicate that conclusions drawn from layer-wise ID estimation may need to be revisited across a wide range of contexts and model classes. Our analysis further suggests that complementing ID-based results with entropy-based measures may offer a more complete picture of the structure of neural representations.

\subsection{Future Work}
\paragraph{Entropy \& ID} While there might be many factors driving the found phenomena in neural representations, we provide empirical evidence that suggests ID estimates are connected to entropy. Based on current evidence, a deeper investigation and formalization of this connection is a promising direction for future research. 

\paragraph{ID Estimators} 
Another future avenue is the design of a new estimator that provides more reliable ID estimates of neural representations. While developing such an estimator is beyond the scope of the present work, we believe our theoretical results offer a useful insight for benchmarking future methods. In particular, layer-wise ID estimates should not increase across the layers of (Lipschitz) neural networks as a minimal requirement to be consistent with theory. This condition is necessary but not sufficient for reliability, since a non-increasing pattern can still differ from the true IDs, but it already rules out several commonly used estimators, as shown in our analysis. We believe this provides a valuable starting point for principled evaluation of ID estimators on neural representations. 

\bibliography{bibliography}

\section*{Checklist}

%

\begin{enumerate}

  \item For all models and algorithms presented, check if you include:
  \begin{enumerate}
    \item A clear description of the mathematical setting, assumptions, algorithm, and/or model. [Yes] Clear descriptions are provided both in the main text and the Supplementary Material.
    \item An analysis of the properties and complexity (time, space, sample size) of any algorithm. [Not Applicable]
    \item (Optional) Anonymized source code, with specification of all dependencies, including external libraries. [Yes] The code is made available in the Supplementary Material.
  \end{enumerate}

  \item For any theoretical claim, check if you include:
  \begin{enumerate}
    \item Statements of the full set of assumptions of all theoretical results. [Yes] A discussion of the full set of assumptions is provided along each theoretical result.
    \item Complete proofs of all theoretical results. [Yes] Complete proofs are provided for all our theoretical results.
    \item Clear explanations of any assumptions. [Yes] We aimed to clearly explain all the mentioned assumptions.   
  \end{enumerate}

  \item For all figures and tables that present empirical results, check if you include:
  \begin{enumerate}
    \item The code, data, and instructions needed to reproduce the main experimental results (either in the supplemental material or as a URL). [Yes] Code, data, and instructions reproduce all experimental results.
    \item All the training details (e.g., data splits, hyperparameters, how they were chosen). [Yes] Details are provided in the Supplementary Material.
    \item A clear definition of the specific measure or statistics and error bars (e.g., with respect to the random seed after running experiments multiple times). [Yes] Clear definitions are provided.
    \item A description of the computing infrastructure used. (e.g., type of GPUs, internal cluster, or cloud provider). [Yes] Details are provided in the Supplementary Material.
  \end{enumerate}

  \item If you are using existing assets (e.g., code, data, models) or curating/releasing new assets, check if you include:
  \begin{enumerate}
    \item Citations of the creator If your work uses existing assets. [Yes]
    \item The license information of the assets, if applicable. [Not Applicable]
    \item New assets either in the supplemental material or as a URL, if applicable. [Not Applicable]
    \item Information about consent from data providers/curators. [Not Applicable]
    \item Discussion of sensible content if applicable, e.g., personally identifiable information or offensive content. [Not Applicable]
  \end{enumerate}

  \item If you used crowdsourcing or conducted research with human subjects, check if you include:
  \begin{enumerate}
    \item The full text of instructions given to participants and screenshots. [Not Applicable]
    \item Descriptions of potential participant risks, with links to Institutional Review Board (IRB) approvals if applicable. [Not Applicable]
    \item The estimated hourly wage paid to participants and the total amount spent on participant compensation. [Not Applicable]
  \end{enumerate}

\end{enumerate}

\clearpage
\appendix
\crefalias{section}{appendix}
\crefalias{subsection}{appendix}
\crefalias{subsubsection}{appendix}
\thispagestyle{empty}

\onecolumn
\aistatstitle{Supplementary Materials: Rethinking Intrinsic Dimension Estimation in Neural Representations}

\section{Neural Networks are Lipschitz Mappings} \label{app:NN_Lipschitz}

In this section, we discuss in more detail the Lipschitz assumption used in our main results. We consider deterministic neural network layers at inference time and assume all weights and scalars are finite. Then, the following holds:

\begin{itemize}
    \item Standard linear or convolutional layers are Lipschitz mappings \citep[Cor.~2.1]{kim2021lipschitz};
    \item Pointwise activations such as ReLU, leaky-ReLU, $\tanh$, sigmoid, softplus, GELU / SiLU are Lipschitz \citep{tsuzuku2018lipschitz};
    \item Pooling operators and residual additions preserve Lipschitzness of the composition \citep{tsuzuku2018lipschitz, bethune2022pay};
    \item Softmax is Lipschitz on $\R^d$ \citep[Prop.~4]{gao2017properties};
    \item Normalization layers such as LayerNorm, BatchNorm, 
    and \mbox{RMSNorm} are Lipschitz \citep{tsuzuku2018lipschitz}.
\end{itemize}
Given that most neural networks are compositions of Lipschitz mappings (compositions of the above components), and given that compositions of Lipschitz mappings remain Lipschitz,  Theorem~\ref{thm:mono-hausdorff} and Theorem~\ref{thm:mono-pointwise} 
apply for such networks.

\paragraph{Not Globally Lipschitz Mappings}
Operations that are discontinuous or not globally Lipschitz are, for example, hard quantization or sign/argmax/top-$k$ gating. Another subtle exception is self-attention. While there are Lipschitz variants such as $L_2$ self-attention, standard (scaled) dot-product self-attention is not globally Lipschitz on unbounded domains \citep{kim2021lipschitz}. Clearly, the same holds for multi-head attention given that it is just a linear map of single self-attention outputs. Nevertheless, if the input space is compact (e.g.\ for bounded inputs), self-attention is Lipschitz on that set \citep{kim2021lipschitz}.
Hence, in this and the other special cases, one could alternatively state the results of Theorem~\ref{thm:mono-hausdorff} and Theorem~\ref{thm:mono-pointwise} on a compact subset of the data domain on which each layer is Lipschitz. The conclusions then hold relative to that subset.

\section{Omitted Proofs and Derivations}\label{app:proofs}

\subsection{Hausdorff Dimension}\label{app:def_haus_dim}
\begin{definition}[Hausdorff measure and dimension \citep{hausdorff1918dimension}]
    Let $s\ge0$ and $\delta >0$ and $E\subset\mathbb R^d$. Define 
\[
 \Hd^s_\delta(E)\ :=\ \inf\Big\{\sum_{i=1}^{\infty} (\operatorname{diam} U_i)^s:\ E\subset\bigcup_{i=1}^{\infty} U_i,\ \operatorname{diam} U_i\le\delta\Big\},
\]
where the infimum is considered with respect to all countable $\delta$-covers $\{U_i\}$ of $E$, and $\operatorname{diam}U:=\sup\{\|x-y\|:x,y\in U\}$. Further, the $s$-dimensional Hausdorff measure is defined by
\[
 \Hd^s(E)\ :=\ \lim_{\delta\downarrow 0}\Hd^s_\delta(E)\ =\ \sup_{\delta>0}\Hd^s_\delta(E).
\]
The Hausdorff dimension of the set $E$ is then defined by
\[
 \dim_H(E)\ :=\ \inf\{s:\ \Hd^s(E)=0\}\ =\ \sup\{s:\ \Hd^s(E)=\infty\}.
\]
\end{definition}

\subsubsection{Monotonicity of the Hausdorff dimension under Lipschitz Maps}

The following Lemma \ref{lem:holder-hausdorff} is a classic result from fractal geometry and demonstrates that the Hausdorff dimension cannot grow under Lipschitz mappings \citep[see, e.g., ][Prop.~2.3 \& Cor.~2.4 (a)]{falconer2013fractal}. We provide a proof for completeness. We start by showing the result for the more general case of Hölder smooth maps ($f$ is \emph{$(L,\alpha)$-Hölder}, $\alpha\in(0,1]$, if \mbox{$\|f(x)-f(y)\|\le L\|x-y\|^\alpha$}) and follow the result for Lipschitz maps ($\alpha=1$) from it.

\begin{lemma}[Hausdorff dimension under Hölder/Lipschitz mappings]
\label{lem:holder-hausdorff}
Let $f:\R^n\to\R^m$ be $(L,\alpha)$-Hölder and $E\subset\R^n$. Then
\[
\dimH\big(f(E)\big)\ \le\ \frac{1}{\alpha}\,\dimH(E).
\]
In particular, if $f$ is Lipschitz ($\alpha=1$), then $$\dimH(f(E))\le \dimH(E).$$
\end{lemma}
\begin{proof}
Fix $s>\dim_H(E)$. Then by definition $\Hd^s(E)=0$, hence for each $\eta>0$ there exists $\delta>0$ with $\Hd^s_\delta(E)<\eta$. This means there is a cover $E\subset\bigcup_i U_i$ with $\operatorname{diam}(U_i)\le\delta$ and
\(\sum_i (\operatorname{diam}U_i)^s<\eta\). Set $t>s/\alpha$ and fix an arbitrary $\delta'>0$. Then choose $\delta\le (\,\delta'/L\,)^{1/\alpha}$. For the cover above, we then have by Hölder continuity,
$\operatorname{diam}\big(f(U_i)\big)\ \le\ L\,\operatorname{diam}(U_i)^\alpha\ \le\ \delta'\,,
$
so $\{f(U_i)\}_i$ is a $\delta'$–cover of $f(E)$. Hence
\[
 \Hd^t_{\delta'}\big(f(E)\big)\ \le\ \sum_i \big(\operatorname{diam} f(U_i)\big)^t
 \ \le\ L^t \sum_i \big(\operatorname{diam}U_i\big)^{\alpha t}
 \ \le\ L^t \sum_i \big(\operatorname{diam}U_i\big)^{s}
 \ <\ L^t\,\eta,
\]
where we used $\alpha t>s$ in the second-to-last  inequality. Since $\eta>0$ was arbitrary, $\Hd^t_{\delta'}(f(E))=0$ for every $\delta'>0$, and therefore
$
 \Hd^t\big(f(E)\big)\ =\ \lim_{\delta'\downarrow 0}\Hd^t_{\delta'}\big(f(E)\big)\ =\ 0.
$
As this holds for all $t>s/\alpha$, we obtain $\dim_H(f(E))\le s/\alpha$. Letting $s\downarrow\dimH(E)$ concludes the proof.
\end{proof}
A similar result for the so-called \textit{Minkowski dimension} can be found in {\citet[Prop.~2.4]{hochman2014lectures}}. Lemma \ref{lem:holder-hausdorff} can be used to show the layer-wise monotonicity of the Hausdorff dimension. Hence, $\dimH$ cannot increase over the layers of any Lipschitz neural network.
\begin{theorem}[Layer-wise monotonicity of Hausdorff dimension]\label{thm:mono-hausdorff}
Let $f_1,\ldots,f_L$ be Lipschitz maps and set $\mu_\ell=(f_\ell)_\#\mu_{\ell-1}$.
Then, for each $\ell \in \{1, \ldots, L\}$,
\[
\dimH\!\big(\supp\mu_{\ell}\big)\ \le\ \dimH\!\big(\supp\mu_{\ell-1}\big).
\]
If $f_\ell$ is only $(L_\ell,\alpha_\ell)$-Hölder, then
$$\dimH(\supp\mu_\ell)\le \alpha_\ell^{-1}\dimH(\supp\mu_{\ell-1}).$$
\end{theorem}
\begin{proof}
Consider Lemma~\ref{lem:holder-hausdorff} 
with $E=\supp\mu_{\ell-1}$ and $f=f_\ell$, and note
$f_\ell(\supp\mu_{\ell-1})=\supp\mu_\ell$.
\end{proof}

\subsection{Pointwise Dimension}\label{app:def_point_dim}

We begin by defining the concept of \textit{pointwise dimension} that was first introduced by \cite{young1982dimension} more formally. For a measure $\mu$, its \textit{upper} and \textit{lower pointwise dimension} at point $x$ are
\[
\overline d_\mu(x)=\limsup_{r\downarrow 0}\frac{\log \mu(\ball(x,r))}{\log r},
\quad
\underline d_\mu(x)=\liminf_{r\downarrow 0}\frac{\log \mu(\ball(x,r))}{\log r},
\]
where $\ball(x,r)$ corresponds to a ball with radius $r$ that is centered around the point $x$.
When the upper and lower limit agree, i.e.\ $\overline d_\mu(x) = \underline d_\mu(x)$, it is also called \textit{pointwise} (or \textit{local Hausdorff}) \textit{dimension} and is denoted by $d_\mu(x)$. Note that the pointwise dimension is defined for a single point instead of the entire dataset, which is why it is sometimes considered a local instead of a global dimension. Nonetheless, it can be considered at multiple points $x$. In particular, $d_\mu(x)$ is said to be \textit{exact dimensional} in case it exists and is $\mu$-a.s.\ independent of the point $x$ (hence, $d_\mu(x)$ equals the same constant for all $x$ $\mu$-a.s.). In this case, it is sometimes denoted by $d_\mu$ {\citep[Def.~3.9]{hochman2014lectures}}.

\subsubsection{Proof of \cref{lem:pointwise-holder}}

\begin{proof}
For small $\rho>0$, $(L,\alpha)$-Hölder gives $f(\ball(x,\rho))\subseteq \ball(y,L\rho^\alpha)$. Setting $r=L\rho^\alpha$,
\[
\nu(\ball(y,r))
=\mu\!\left(f^{-1}(\ball(y,r))\right)
\ \ge\ \mu\! \big(\ball(x,\rho)\big).
\]
Taking logs and dividing by $\log r<0$ (fulfilled for a sufficiently small $r$) reverses the inequality:
\[
\frac{\log \nu(\ball(y,r))}{\log r}
\ \le\
\frac{\log \mu(\ball(x,\rho))}{\log r}
=
\frac{\log \mu(\ball(x,\rho))}{\alpha\log \rho+\log L}. 
\]
As $r\downarrow0$, we have $\rho\downarrow0$ and $\log \rho\to-\infty$, so the additive constant $\log L$ is negligible in $\limsup$/$\liminf$. Hence, for the upper pointwise dimension (same logic applies to the lower pointwise dimension), we get 
\[
\overline d_\nu(y)
=\limsup_{r\downarrow 0}\frac{\log \nu(B(y,r))}{\log r}
\;\le\;
\limsup_{\rho\downarrow 0}\frac{\log \mu(B(x,\rho))}{\alpha\log \rho+\log L}
=\frac{1}{\alpha}\,\limsup_{\rho\downarrow 0} \frac{\log \mu(B(x,\rho))}{\log \rho}
=\frac{1}{\alpha}\,\overline d_\mu(x).
\]
For Lipschitz maps ($\alpha=1$), we obtain $\overline d_\nu(y)\le\overline d_\mu(x)$ and $\underline d_\nu(y)\le\underline d_\mu(x)$, which concludes the proof.
\end{proof}

\subsection{Invariance of Hausdorff and Pointwise Dimensions Under Bi-Lipschitz Mappings}\label{app:bilipschitz}

The first part of the following result is another classic result from fractal geometry that can be found in  \citet[Cor.~2.4 (b)]{falconer2013fractal}. The second part that is about the (upper/lower) pointwise dimension is discussed in \citet{hidaka2013estimation}. 

\begin{proposition}[Invariance under Bi-Lipschitz mappings]\label{prop:bilip}
If $f$ is bi-Lipschitz on $E\subset\R^n$, then $\dimH(f(E))=\dimH(E)$. 
If moreover $\nu=f_\#\mu$ and $f$ is bi-Lipschitz on $\supp\mu$, then $\overline d_\nu(f(x))=\overline d_\mu(x)$ and $\underline d_\nu(f(x))=\underline d_\mu(x)$ for $\mu$-a.e.\ $x$.
\end{proposition}

\begin{proof}
Applying Lemma~\ref{lem:holder-hausdorff} to the Lipschitz function $f: E \to \mathbb{R}^m$ yields $\dimH(f(E))\leq\dimH(E)$. Due to the bi-Lipschitzness, $f^{-1}: f(E) \to E$ is also Lipschitz. Hence, by Lemma~\ref{lem:holder-hausdorff} we get that $\dimH(E)\leq\dimH(f(E))$. This proofs  $\dimH(E) =\dimH(f(E))$ for bi-Lipschitz $f$. For the (upper/lower) pointwise dimensions, apply Lemma~\ref{lem:pointwise-holder} to both $f$ and $f^{-1}$ using the same logic from above.
\end{proof}

\subsection{MLE and TwoNN target the pointwise dimension}\label{app:mle_2nn_targets}

Let $\tilde{M}$ be a $d$-dimensional manifold and let $Y_{1},\dots,Y_{n}\in \tilde{M}$ be i.i.d.\ with probability
measure $\tilde{\mu}$. For simplicity, we use a slightly different notation in the following derivation (compared to other sections), with $d$ and $D$ ($d\!\ll\!D$) denoting the dimension of the manifold and the ambient space, respectively. The observed sample $\{X_j\}_{j=1}^{n}$ is its (smooth) embedding
$X_{j}:=g(Y_{j})\in\R^{D}$, with a continuous and sufficiently smooth mapping $g$ as in \citet{levina2004maximum}. Let $M := g(\tilde{M})$ denote the embedded data manifold, and $\mu$ its induced probability measure.

\paragraph{Local Model (Homogeneous PPP)}
Fix a point $x\in M$ for which the pointwise dimension exists, $d_{\mu}(x)=d$.
Assume that in a neighborhood of $x$, $\mu$ has a density
$\kappa$ with respect to the Riemannian volume on $M$, with $\kappa$ continuous at $x$ and $0<\kappa(x)<\infty$.
Under these conditions, the standard
Binomial-to-Poisson coupling \citep{penrose2003weak, penrose2013limit} implies that, at sufficiently small scales around $x$, the sample $\{X_j\}_{j=1}^{n}$ can be approximated by a \emph{homogeneous Poisson point process} (PPP) in the tangent space $\mathbb{R}^d$ with \emph{intensity} $\lambda_n=n\,\kappa(x)$, meaning the expected number of points in a set equals $\lambda_n$ times its $d$-dimensional volume. In particular, for small $r>0$, the count
$
N(r,x):= \sum_{j=1}^n \mathbf 1\{X_j\in \ball(x,r)\}
$
is well-approximated by
\mbox{$N(r,x) \sim \mathrm{Poisson}\!\big(\lambda_n\,\omega_d\, r^{d}\big),$}
where $\omega_d$ is the volume of the unit ball in $\mathbb R^d$.

\paragraph{Levina-Bickel MLE}
Let $T_i(x)$ be the distance from $x$ to its $i$-th nearest neighbor. Under the local PPP model, conditional on the distance $T_k(x)$ to the $k$-th neighbor, the ratios $U_i = T_i(x) / T_k(x)$ for $i=1, \dots, k-1$ are the order statistics of $k-1$ i.i.d. random variables drawn from a distribution with CDF $\mathfrak F(u)=u^d$ and corresponding PDF $\mathfrak f(u|d)=du^{d-1}$ for $u\in[0,1]$. The joint log-likelihood of these order statistics is $\ell(d)= C + \sum_{i=1}^{k-1} \log(d{u_i}^{d-1})$, where $C$ is a constant independent of $d$. Maximizing $\ell(d)$ w.r.t.\ $d$, yields the MLE from \cite{levina2004maximum}:
\[
\widehat d_{\mathrm{MLE}}(x)
\ =\ \left[\frac{1}{k-1}\sum_{i=1}^{k-1}\log\frac{T_k(x)}{T_i(x)}\right]^{-1}.
\]
As $n\to\infty$, $k\to\infty$ and $k/n\to0$ so that $r=T_k(x)\downarrow0$, the PPP small-scale approximation becomes exact in the limit. Because $\widehat d_{\mathrm{MLE}}(x)$ is asymptotically unbiased and its variance scales as $O(1/k)$ \citep{levina2004maximum}, it follows that $\widehat d_{\mathrm{MLE}}(x)\to d$ in probability. Note that by defining $M$ as a $d$-dimensional manifold, we assumed exact dimensionality ($d_{\mu}(x)=d$ for almost all $x \in M$). However, even in more general spaces with varying pointwise dimension, as the local PPP approximation becomes exact, we obtain the localized result that $\widehat d_{\mathrm{MLE}}(x)\to d_\mu(x)$ in probability.

\paragraph{TwoNN}
In the special case $k=2$, the ratio $\rho(x):=T_2(x)/T_1(x)\in[1,\infty)$ satisfies $\log\rho(x)\sim\mathrm{Exp}(d)$ and hence $\rho$ is Pareto$(d)$:
\[
\mathfrak f(\rho\,|\,d)=d\,\rho^{-(d+1)},\qquad \mathfrak F(\rho\,|\,d)=1-\rho^{-d}\quad(\rho\ge 1),
\]
see \citet[Eqs.\ (5) \& (6)]{facco2017estimating}.
Treating $\{\rho(x_j)\}$ as approximately independent gives the pseudo-log-likelihood
$\ell(d)=n\log d-(d+1)\sum_{j=1}^n\log\rho(x_j)$, whose maximizer is
\[
\widehat d_{\mathrm{TwoNN}}
=\left[\frac{1}{n}\sum_{j=1}^n \log\frac{T_2(x_j)}{T_1(x_j)}\right]^{-1}.
\]
Alternatively, the estimation can be based on linear regression (both approaches are asymptotically equivalent). The regression form used in TwoNN follows from the Pareto distribution-based identity
\[
-\log\!\bigl(1-\mathfrak F(\rho\,|\,d)\bigr)=d\,\log \rho,
\]
see \citet[Eqs.\ (7)]{facco2017estimating}. So fitting a straight line (passing through the origin) to $\{\bigl(\log\rho_j,\ -\log(1-\widehat{\mathfrak{F}}_n(\rho_j))\bigr)\}_{j=1}^{n}$ estimates the slope $d$ \citep{facco2017estimating}. Under exact dimensionality we have that $d_\mu(x)=d$ for $\mu$-a.e.\ interior $x$. In that case, as $n\to\infty$ and $r=T_2(x)\downarrow 0$, the PPP small-scale approximation becomes exact, justifying $\log\rho\sim\mathrm{Exp}(d)$. Since $\mathbb{E}[\log\rho]=d^{-1}$, applying the the weak law of large numbers to the sample mean, followed by the continuous mapping theorem, yields $\widehat d_{\mathrm{TwoNN}}\to d$ in probability.

\subsection{Exact IDs of LLM Embeddings \& Representations}\label{app:exact-id-llms}

As depicted in \cref{fig:word_emd_llms}, each token in a prompt is mapped to a specific point in the latent space of the LLM embeddings. While these embeddings vary between LLMs, and the number of unique points that can be reached in the latent space of the embeddings depends on the size of the vocabulary of the respective LLM, this number is \textit{finite} for all LLMs. Using the fact that the ID of any finite or countable set of points is zero, the lemma below shows that the ID of LLM embeddings are also zero in a strict topological sense. Considering prompts with finite length, the lemma extends this result to the ID of hidden-layer representations in LLMs. 

\begin{lemma}[ID of LLM representations]\label{lem:zero-dim-llm}

Let $\mathcal V$ be a finite vocabulary, and let $\bigcup_{n \ge 1} \mathcal{V}^n$ be the countable set of all prompts, where each prompt $x$ has an arbitrary finite length $|x|=n \in\mathbb{N}$.
Let $e:\mathcal V\to\mathbb R^{d_0}$ be the token embedding map, then the token embedding set $e(\mathcal V)$ is finite as $\mathcal V$ is finite. Further, denote $S_0 = \{ (x, t) : x \in \bigcup_{n \ge 1} \mathcal{V}^n, t \in \{1, \dots, |x|\} \}$ to be the set of all prompts and positions.
Consider a deterministic, measurable LLM whose $\ell$-th layer representation at position $t$ is given by
\[
F_\ell:\ \bigcup_{n\ge1}\mathcal V^n \times \mathbb N \ \longrightarrow\ \mathbb R^{d_\ell},\qquad
(x,t)\mapsto F_\ell(x,t),
\]
defined for $t\in\{1,\dots,|x|\}$. For $\ell\ge1$ define the set of attainable representations in layer $\ell$ by
\[
S_\ell \;:=\; \bigl\{F_\ell(x,t): x\in \textstyle\bigcup_{n\ge1}\mathcal V^n,\ t\in\{1,\dots,|x|\}\bigr\}.
\]
Then for every $\ell\in\{1,\dots,L\}$:
\begin{enumerate}
\item (Hausdorff dimension) $S_\ell$ is countable and $\dim_H(S_\ell)=0$.
\item (Pointwise dimension) If $\mu_0$ is any probability measure supported on $S_0$ and
$\mu_\ell=(F_\ell)_\#\mu_{0}$, then $\mu_\ell$ is purely atomic and, for $\mu_\ell$-a.e.\ $y$,
\[
\underline d_{\mu_\ell}(y)=\overline d_{\mu_\ell}(y)=0.
\]
Hence each $\mu_\ell$ is exact-dimensional with $d_{\mu_\ell}(y)=0$ for $\mu_\ell$-a.e.\ $y$.
\end{enumerate}
\end{lemma}

\begin{proof}
The set $\bigcup_{n\ge1}\mathcal V^n$ is countable (countable union of finite sets), and for each $x$, the index set
$\{1,\dots,|x|\}$ is finite. Thus, the domain $S_0=\{(x,t)\}$ is countable, and its image $S_\ell$ under any function is countable.
Further, any countable subset of $\mathbb R^{d_\ell}$ has a Hausdorff dimension of zero. This proves the first claim.

For the second claim, note that $\mu_0$ is purely atomic as $S_0$ is countable. The pushforwards of purely atomic measures under any (measurable) map are again purely atomic and supported on $S_\ell$, which was shown to be countable above. For any point $y$ (an atom) with mass $c=\mu_\ell(\{y\})>0$, we have that $\mu_\ell(B(y,r))\ge c$ for any arbitrarily small $r>0$. Combining this with the fact that the (lower/upper) pointwise dimension is bounded below by zero, we get that
$$
0\le\lim_{r\downarrow0}\frac{\log \mu_\ell(B(y,r))}{\log r}
\le\lim_{r\downarrow0}\frac{\log c}{\log r}=0,
$$
which shows that both lower and upper pointwise dimensions are zero at $\mu_\ell$-a.e.\ $y$.
\end{proof}

\subsubsection{Conceptual difference between Text and Image Representations}\label{app:diff-img-text}

While it might seem natural to extend the above reasoning from text to image representations, the following discussion aims to clarify why \cref{lem:zero-dim-llm} cannot be transferred to image representations.
For token-based inputs, the key issue is that one cannot meaningfully and smoothly interpolate between different token embeddings. This is not just due to the finiteness of the input, but to the inherently discrete nature of words and subword units, which yields a finite and separated set of token embeddings.

In contrast, for images it is conceptually natural to think about smooth interpolation between inputs, reflecting that the underlying light spectrum is continuous in reality. Although this continuum is discretized when images are recorded and stored with finite precision, this is typically viewed as a technical approximation rather than a property of the underlying object. Thus, images can be viewed as forming a high-dimensional continuous manifold embedded in a discrete space, where the discreteness arises from digital representation rather than from the underlying object itself, unlike in the token-based case.

Therefore, the difference between images and text with respect to the manifold hypothesis stems from the fundamentally different nature of pixels and token embeddings. Conceptually, smooth interpolation is natural for pixels but not for tokens. \cref{lem:zero-dim-llm} is aimed at making this intuition mathematically precise.

\section{Additional Experiments and Experimental Details}\label{app:add_exp}

\paragraph{Computing Environment} Experiments involving the extraction of layer-wise representations for the LLMs and ViTs were performed on an NVIDIA Tesla T4 GPU. The CNN experiments can additionally be run using Apple's Metal Performance Shaders (MPS) backend on a MacBook with at least 16\,GB of RAM. In our experimental setup, extracting layer-wise representations for a single model and computing each metric required less than one hour. The code is available at \url{https://github.com/rickmer-schulte/rethinking-neural-id}.

\subsection{Layer-wise Analysis}\label{app:layerwise_analysis}

\subsubsection{Layer-wise Analysis of LLM Representations}

For the layer-wise neural representation analysis of LLMs, we follow along the investigation of \citet{cheng2025emergence}.
Similar to their analysis, we base ID-estimation on 10k prompts from the popular wikitext dataset \citep{merity2017pointer}. For each prompt, we extract layer-wise representations at the final position as described in \cref{sec:id_llms}. We do this for the pretrained models \mbox{Llama-3.1-8B} \citep{grattafiori2024llama}, \mbox{Mistral-7B-v0.3} \citep{jiang2023mistral7b}, and \mbox{Pythia-6.9B} \citep{biderman2023pythia}. Pretrained weights for each of these models are obtained from Hugging Face. 

Similar to several other LLM-based analyses, ID estimation is based on the Gride ID estimator \citep{denti2022gride} using the \textit{DADApy} implementation \citep{dadapy2022} with default parameter settings, such as $\text{range\_max}=64$. The latter indicates that the $64^{th}$ NN distance is the maximum NN distance that is involved in the ID estimation.

\subsubsection{Layer-wise Analysis of CNN Representations}
For the layer-wise neural representation analysis of convolutional neural networks (CNNs), we follow along the investigation of \citet{ansuini2019intrinsic}. We consider several classic model architectures such as \textit{Alexnet} \citep{krizhevsky2012alexnet}, \textit{VGG} \citep{simonyan2014vgg}, \textit{ResNet} \citep{he2016resnet} with pretrained weights (pretrained on ImageNet \citep{deng2009imagenet}) obtained from the PyTorch library \textit{torchvision} \citep{paszke2019pytorch}. The addition of ``(BN)'' for the VGGs of \cref{fig:id_fig} indicates that models incorporate Batch Normalization layers. 

The layer-wise neural representations are obtained by passing a set of images through each pretrained model and extracting the corresponding representations from the layers. Each dataset consists of 500 samples from the seven largest categories of ImageNet. Further details can be found in \citet{ansuini2019intrinsic}. In \cref{fig:id_fig,fig:id_classes,fig:id_vs_emb_dim,fig:nn_dist_cnns,fig:avg_cos_sim_cnns,fig:avg_length_reps_cnns,fig:id_vs_entropy}, the point estimates correspond to the mean of the respective estimates on the seven datasets, and the error bars to the corresponding standard deviations. The estimated intrinsic dimensions are obtained via the TwoNN ID estimator \citep{facco2017estimating}. In each of these plots, the x-axis denotes the relative instead of the absolute depth of each model layer to facilitate visual comparison between models with varying numbers of layers. 

The CNN-based layer-wise analysis of representations concerning other metrics than the ID is performed along four representative examples of the CNNs in \cref{fig:id_fig}. In the following sections, we conduct experiments analogous to the ones found for LLMs in the main text, along the example of the four pretrained CNN models \textit{Alexnet}, \textit{VGG-16}, \textit{ResNet-18}, and \textit{ResNet-34}.

\subsubsection{Layer-wise Analysis of ViT Representations}

For the analysis of vision transformers (ViTs), we extract layer-wise representations from a standard ViT (google/vit-base-patch16-224) \citep{wu2020visual} and two DINOv3 models of different sizes (facebook/dinov3-vitb16-pretrain-lvd1689m; facebook/dinov3-vitl16-pretrain-lvd1689m) \citep{simeoni2025dinov3} based on 5k images from ImageNet and conducted a layer-wise analysis analogous to our previous investigation. In the spirit of our previous transformer-based analysis, we extract the layer-wise representation of the CLS token for each image, given that all information necessary for prediction tasks is encoded in it. Pretrained weights for each of the ViT models are obtained from Hugging Face.

\subsection{Metrics}\label{app:metrics}

\paragraph{Von Neumann Entropy} In \cref{sec:other_metrics} we considered the \textit{von Neumann entropy} of the distribution of eigenvalues (of the Gram matrix) of layer-wise representations.
The entropy can be computed layer-wise. Let $Z\in \mathbb{R}^{n\times d}$ be a latent representation in a hidden layer based on $n$ observations with $d$-dimensional neural representation. First, define the corresponding Gram matrix 
$Q = ZZ^\top \in \mathbb{R}^{n \times n}$. Since Q is symmetric positive semidefinite, we can normalize it by its trace to obtain 
$$\tilde{Q}=\frac{Q}{\operatorname{tr}(Q)},$$
which satisfies $\operatorname{tr}(\tilde{Q})=1$. The eigenvalues $p_1, \ldots, p_n$ of $\tilde{Q}$ are given by
$$
    p_i = \frac{\lambda_i(Q)}{\operatorname{tr}(Q)}, \quad\quad i=1, \ldots, n,
$$
where $\lambda_1(Q),\ldots,\lambda_n(Q)$ denote the eigenvalues of $Q$.
Then the \textit{von Neumann entropy} of $\tilde{Q}$ is defined as $
  S(\tilde{Q}) = -\sum_{i=1}^n p_i
  \log\!\left( {p_i} \right)$, with the convention that $0\,\log (0)=0$. Because only the nonzero eigenvalues contribute, this can equivalently be written as
\begin{equation}
  S(\tilde{Q}) = -\sum_{i=1}^r p_i
  \log\!\left( {p_i} \right),
  \label{eq:vN_entropy}
\end{equation}
where $r$ denotes the rank of $Q$ with $r = \operatorname{rank}(Q) \leq \min(n,d)$. Note that this is a special type of so-called \textit{matrix-based entropy} 
\citep{giraldo2014measures} that was also considered in the analysis of \citet{skean2025layer}. Considering the normalized singular values 
instead of the eigenvalues 
and exponentiating the corresponding results in \eqref{eq:vN_entropy}, one obtains the so-called \textit{effective rank} of $Z$ \citep{roy2007effective}. In our experiments, we computed the von Neumann entropy based on mean-centered representations to capture the spread of the representations in latent space independent of translations.
In general, point estimates correspond to the mean of the entropy estimates, and the error bars or shaded areas are based on the standard deviations in the entropy-related plots.

\paragraph{Other Metrics} Our experiments also cover a layer-wise analysis of the length and pairwise cosine similarity of representations. The length (or size) of the representations is measured by $L_2$-distances to the origin of the latent space for each representation. The cosine similarity is computed between and averaged over pairs of layer-wise representations, measuring how closely aligned the representations are over several hidden layers of the model. Given the high number of samples and, therefore extremely high number of pairs in the case of the LLM-based analysis, we adapt a block-wise analysis of pairs to be more memory efficient.

\subsection{Intrinsic Dimension Estimator Analysis}\label{app:id_est_analysis}
In \cref{fig:id_fig2}, we investigate the accuracy of the two ID estimators, TwoNN and MLE. In order to obtain datasets with known intrinsic dimension, we sample 5k data points uniformly distributed on a $d_{\mathcal{M}}$-hyperball with varying true ID $d_{\mathcal{M}}$. For each true intrinsic dimension, we sample 20 datasets and estimate the IDs via MLE and TwoNN on each of these. 

\cref{fig:id_fig2} and \cref{fig:id_bias_multimle} depict both the average over these 20 ID estimates (lines) and the related 95\% CI.
For the sampling and subsequent ID estimation, we use the \textit{scikit-dimension} library \citep{bac2021scikit}. \cref{fig:id_fig2} shows a strong negative bias for the two estimators that is growing with increasing true intrinsic dimension. While \cref{fig:id_fig2} uses MLE with $k=20$, \cref{fig:id_bias_multimle} shows that the negative bias is persistent also for other choices of nearest neighbors $k$. 

Similar to \cref{fig:id_vs_emb_dim}, we also demonstrate in a controlled study setup (with known ID and datasets sampled as described above) that the ambient space dimension does not seem to have strong effects on the TwoNN and MLE estimates. For a fixed true ID of $50$, the two estimators still exhibit a negative bias, but are largely invariant to changes in the size of the ambient space dimension, as can be seen in \cref{fig:id_vs_amb_dim}.

\begin{figure}[htbp]
    \centering
    \includegraphics[width=0.54\textwidth]{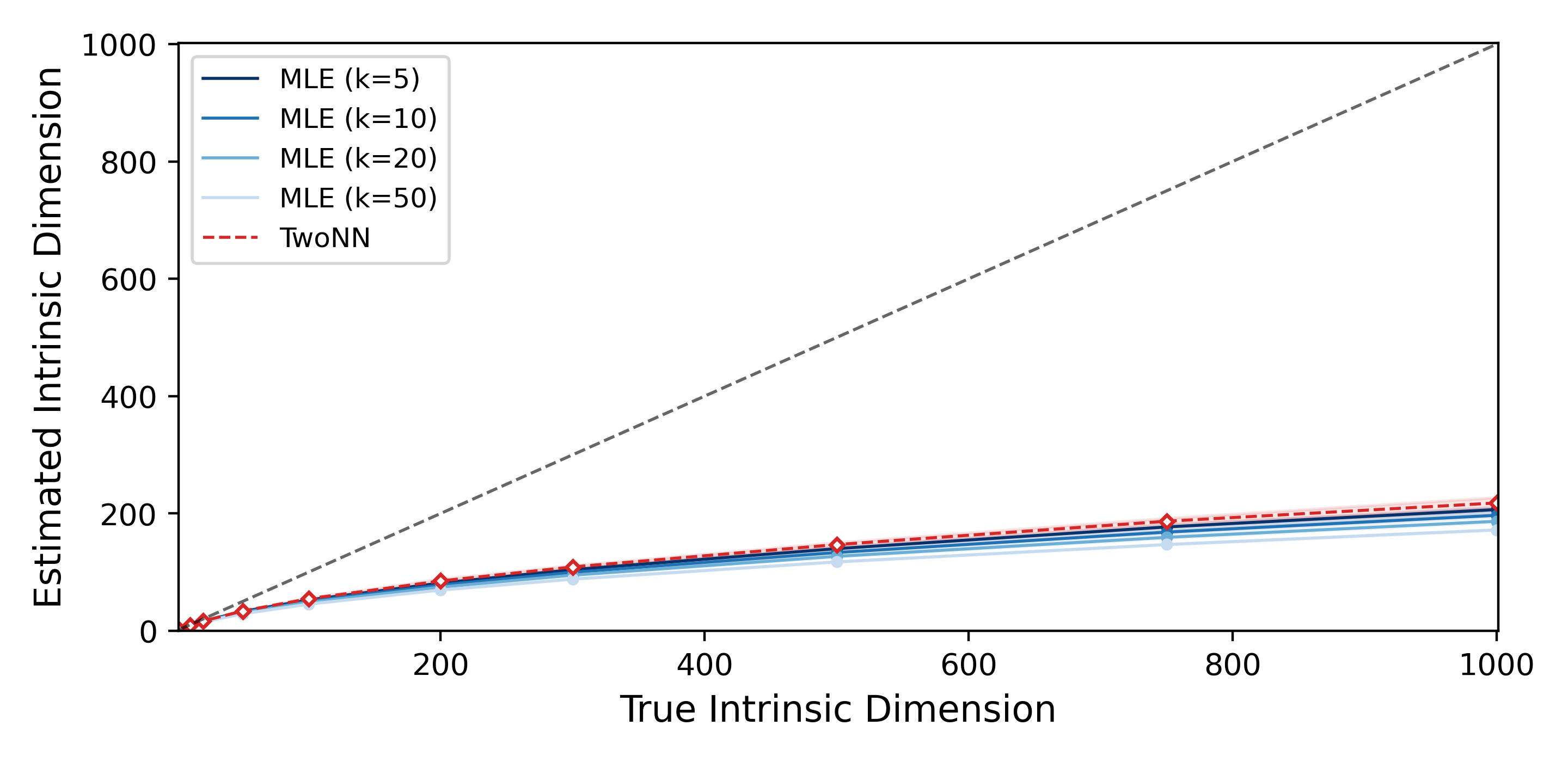}
    \caption{Estimated vs.\ True ID: Estimated IDs using TwoNN and MLE (different $k$) of datasets with varying true ID. Each dataset consists of 5k data points uniformly distributed on a $d_{\mathcal{M}}$-hyperball with varying true ID $d_{\mathcal{M}}$. 95\% CI are computed based on 20 ID estimates. Both estimators exhibit strong negative bias with increasing $d_{\mathcal{M}}$.}
    \label{fig:id_bias_multimle}
\end{figure}

\begin{figure}[htbp]
    \centering
    \includegraphics[width=0.5\textwidth]{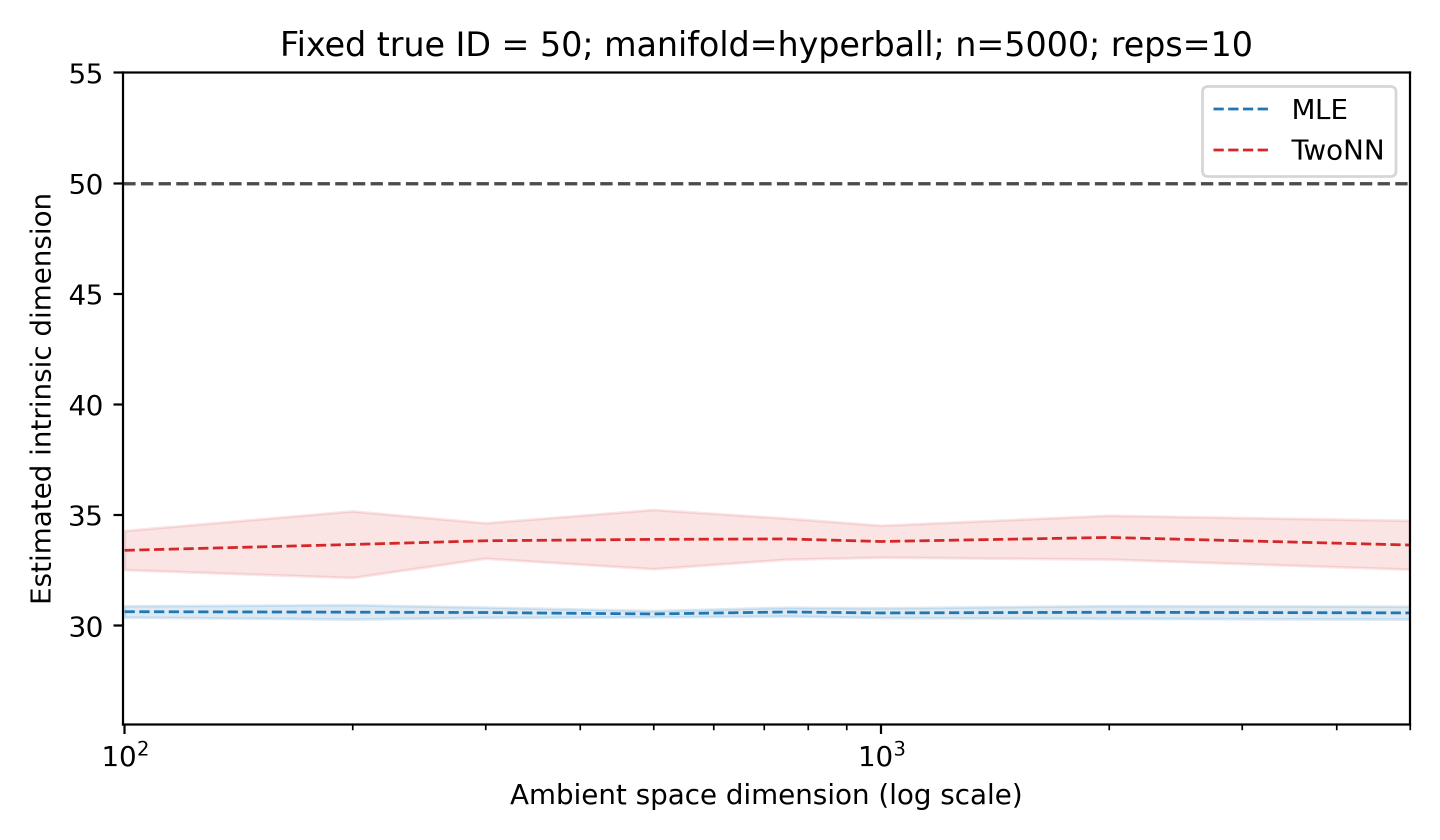}
    \caption{Estimated ID vs.\ Ambient Space Dim.: Estimated IDs using TwoNN and MLE ($k=20$) of datasets with varying ambient space dimension $d$. Each dataset consists of 5k data points uniformly distributed on a $d_{\mathcal{M}}$-hyperball with fixed to $d_{\mathcal{M}}=50$. 95\% CI are computed based on 10 ID estimates. Ambient space dimension does not strongly impact ID estimates.}
    \label{fig:id_vs_amb_dim}
\end{figure}

\subsection{ID vs.\ Ambient Space Dimension}\label{app:id_amb_analysis}
In \cref{sec:other_metrics}, we explored what might be underlying factors that drive the commonly found layer-wise ID patterns. Along with this analysis, we also conduct an experiment comparing the layer-wise ID patterns with the layer-wise embedding dimension in different pretrained models. The result in \cref{fig:id_vs_emb_dim} shows that the two patterns are relatively different, indicating that the ambient space dimension itself does not seem to be the core driving factor for the commonly found ID patterns.

\begin{figure}[htbp]
    \centering
    \includegraphics[width=1\textwidth]{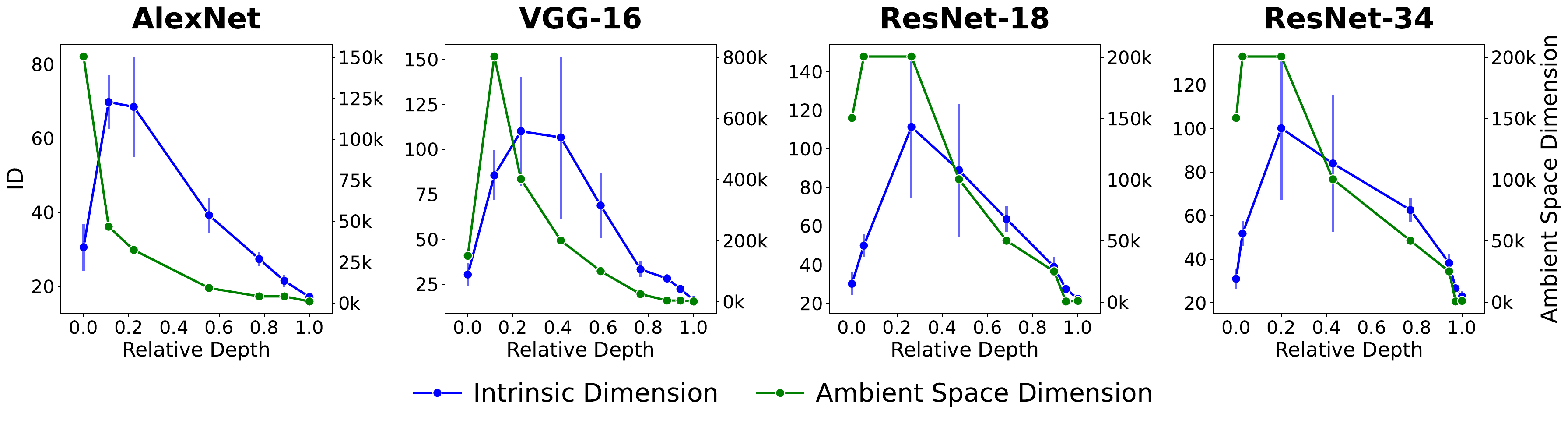}
    \caption{Layer-wise comparison of estimated intrinsic dimensions (left y-axis) vs.\ ambient space dimensions (right y-axis) of neural representations from different pre-trained convolutional architectures. The x-axis shows the relative depth of the model layers. The ambient space dimension corresponds to the width of each layer.}
    \label{fig:id_vs_emb_dim}
\end{figure}

\subsection{Class-specific Intrinsic Dimension Estimates}\label{app:IDs_per_class}

Analogous to \cref{fig:id_classes}, we also plot the class-specific ID estimates for the other pre-trained convolutional architectures in \cref{fig:id_per_class_all_models}. The estimated IDs in each plot are computed separately for the layer-wise representations corresponding to images from the seven largest classes of the ImageNet \citep{deng2009imagenet} dataset. Similar to \cref{fig:id_classes}, also the class-specific ID patterns in case of the other convolutional architectures show an increase in estimated IDs over the layers. ID estimates for layer-wise representations of the vision models are again obtained using the TwoNN estimator.

\begin{figure}[!htb]
    \centering
    \includegraphics[width=1\textwidth]{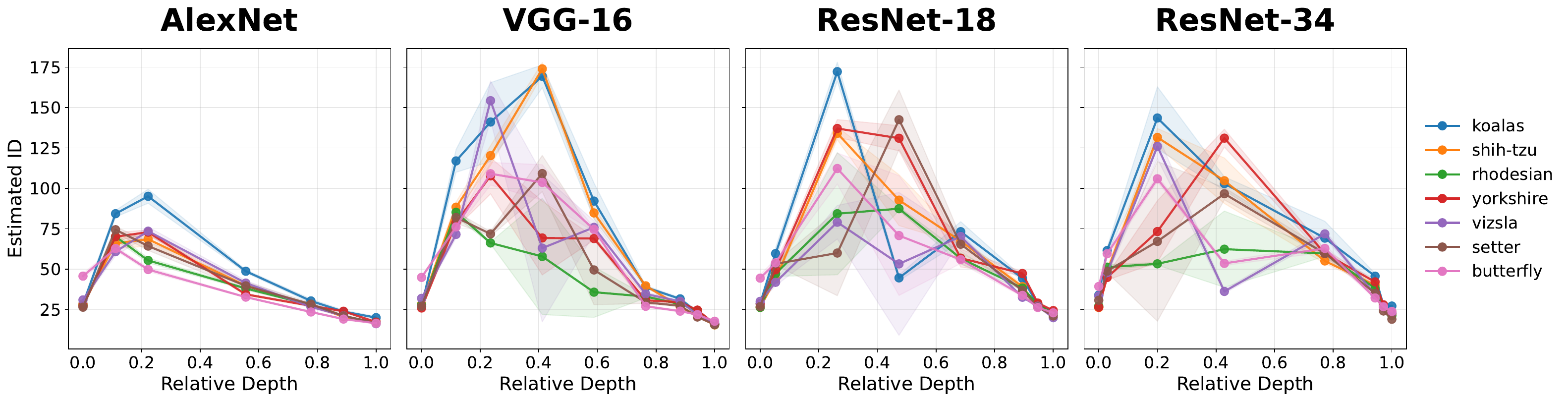}
    \caption{Class-specific estimated IDs of layer-wise representations from various pre-trained convolutional architectures separately for images with different categories (colors) from ImageNet. The x-axis shows the relative depth of model layers, and shaded areas show the estimated standard errors. The first plot is adapted from \citet{ansuini2019intrinsic}.} 
    \label{fig:id_per_class_all_models}
\end{figure}

\subsection{Average k-NN Distance Analysis}\label{app:knn_distances}

We also investigate the average NN distances of layer-wise representations for the pre-trained convolutional architectures. Given that for these types of models, usually the MLE and TwoNN estimators are used for ID estimation, we consider the consecutive NN distances that are involved in their respective calculations. Therefore, \cref{fig:nn_dist_cnns} depicts the average $1^{st}$ to $5^{th}$-NN distances $(k=1, \ldots,5)$ of layer-wise representations from different pretrained vision models. 

The layer-wise NN distances differ in both their shapes and overall magnitude between the models. The fact that average NN distances are much higher for the VGG model can, to some extent, be explained by the much larger ambient space dimensions (width of layers) for this model compared to the others. As can be seen in \cref{fig:id_vs_emb_dim}, the representations in intermediate layers of the VGG-16 model are about 400k-800k dimensional (compared to about 50k-200k dimensional intermediate-layer representations in the other models). Thus, k-NN distances grow given that points become increasingly separated in high dimensions (cf.~\cref{sec:amb_dim}). Interestingly, however, for all models the $k$-th NN distances in \cref{fig:nn_dist_cnns} grow and shrink by similar amounts over the layers. By a similar reasoning as in \cref{sec:nn_dist}, this explains the shape of the layer-wise ID patterns for the vision models (cf.~\cref{fig:id_fig}).
\begin{figure*}[!htb]
    \centering
    \includegraphics[width=1\textwidth]{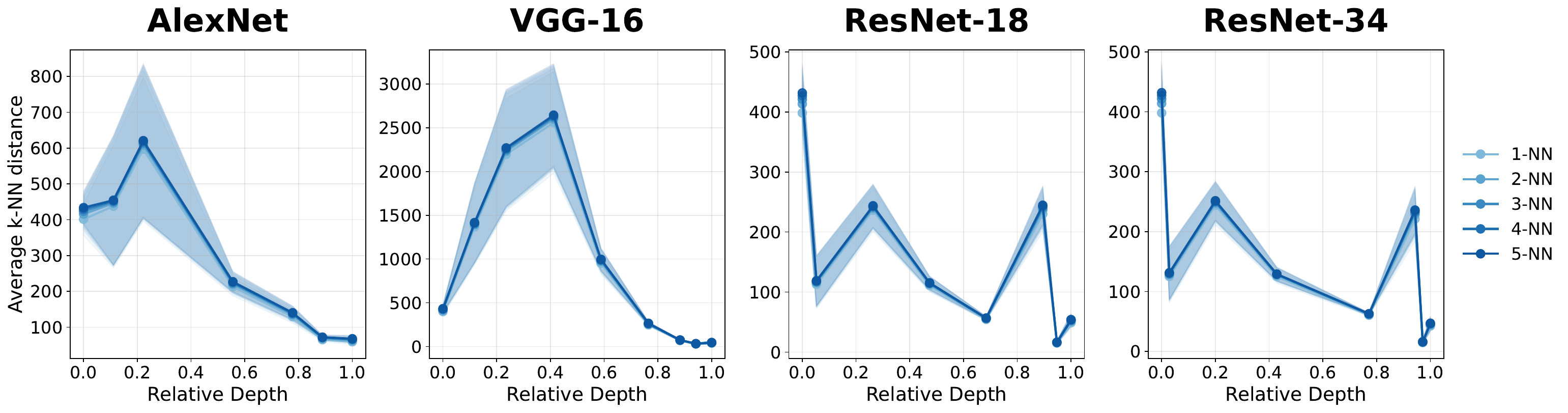}
    \caption{Average k-NN distances of representations over the layers of various pre-trained convolutional architectures. Averages over the $1^{st}$ to $5^{th}$-NN distances are depicted and displayed with a corresponding color. Shaded areas correspond to twice the standard deviation.}
    \label{fig:nn_dist_cnns}
\end{figure*}

\subsection{Cosine Similarity Analysis}\label{app:cos_sim}

\paragraph{LLMs} The extended results of the analysis of pairwise cosine similarity of layer representations from LLMs can be found in \cref{fig:avg_cos_sim_llms}. The latter includes the last layer of the three models. As described in \cref{rm:last_layer}, the last layer representations are subject to a final Layer Norm transformation, which can induce drastic changes in the last layer. The latter is likely the reason for the drop in pairwise cosine similarity in the last layer, which is especially apparent for the pythia model. However, the key insight from \cref{fig:avg_cos_sim_llms} for our analysis is that the layer-wise cosine similarity patterns are very different compared to the layer-wise ID and NN distance patterns depicted in \cref{fig:ids_llms} and \cref{fig:nn_dist_llms}, respectively. Accordingly, varying cosine similarities between representations cannot be the key driving force behind observed ID patterns, in line with the conclusion in \cref{sec:cos_sim}.
\vspace*{3mm}
\begin{figure*}[!htb]
    \centering
    \includegraphics[width=0.6\textwidth]{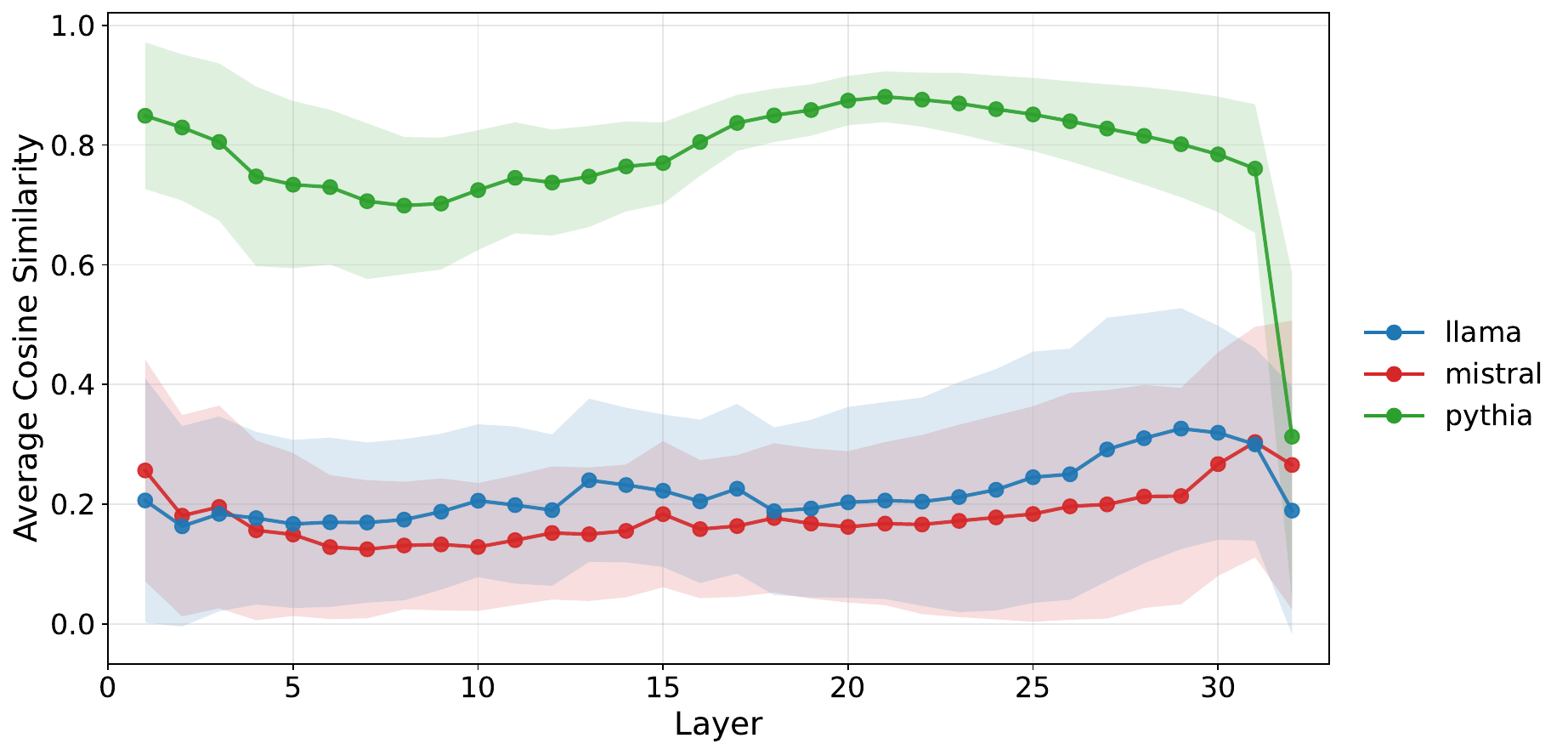}
    \caption{Average cosine similarity between layer-wise representations of the LLM models (llama, mistral, and pythia). The shaded area band represents twice the standard deviation.}
    \label{fig:avg_cos_sim_llms}
\end{figure*}
\paragraph{ViTs} We also investigate the pairwise cosine similarity of layer representations from ViTs. The results can be found in \cref{fig:avg_cos_sim_vits}. Similar to the LLM-based analysis, besides a drastic change in the last layer, pairwise cosine similarity estimates generally do not vary a lot across model layers. Moreover, layer-wise patterns seem very different from the respective layer-wise ID patterns, which are shown in \cref{fig:id_vs_entropy_vits}.
\begin{figure*}[!htb]
    \centering
    \includegraphics[width=0.65\textwidth]{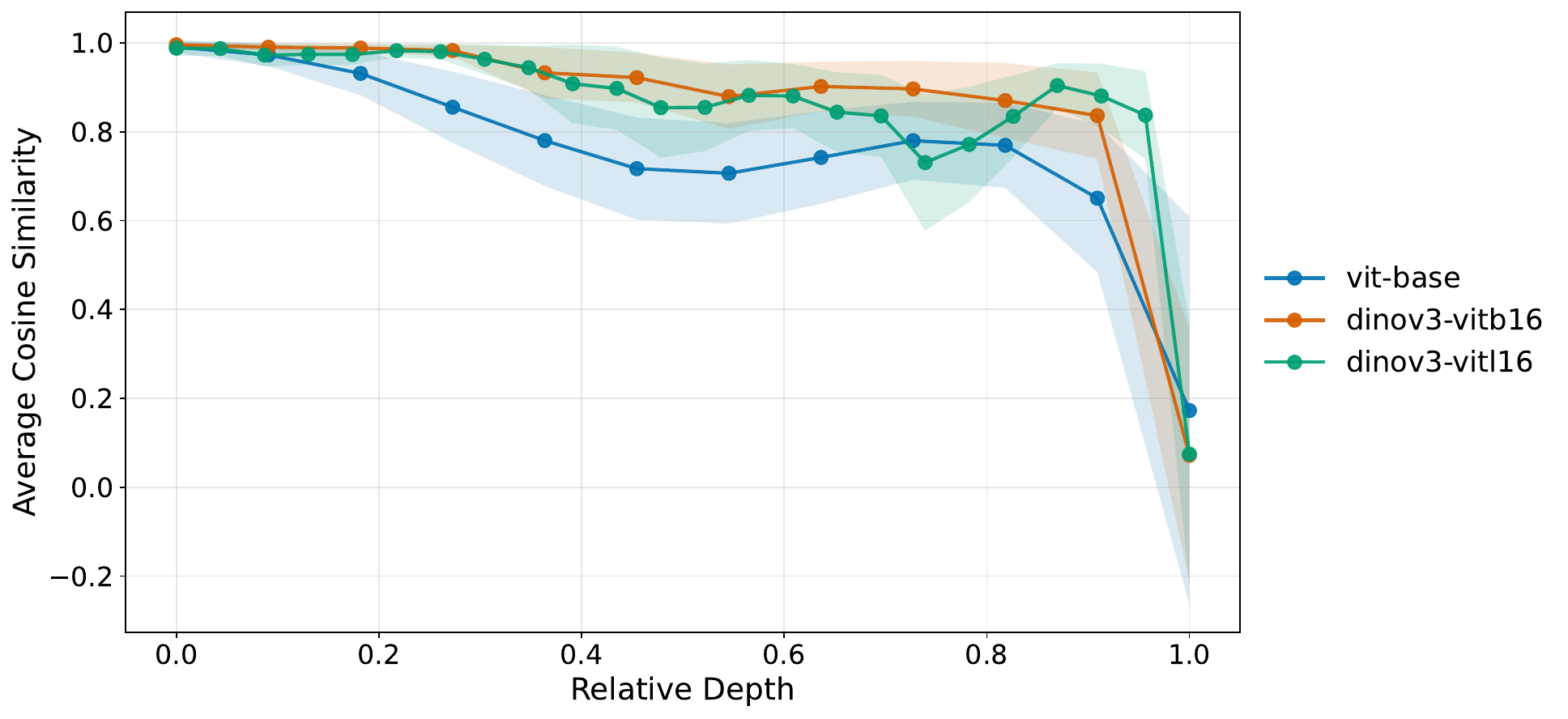}
    \caption{Average cosine similarity between layer-wise representations for different pre-trained ViTs. The shaded area band represents twice the standard deviation.}
    \label{fig:avg_cos_sim_vits}
\end{figure*}
\paragraph{CNNs} We also investigate the pairwise cosine similarity of layer representations from the CNNs. The results for the four pre-trained convolutional architectures can be found in \cref{fig:avg_cos_sim_cnns}. Interestingly, there seems to be a layer-wise pattern of pairwise cosine similarity that is common to all models, which might be of interest on its own. However, for our analysis, the most important insight of \cref{fig:avg_cos_sim_cnns} is that the layer-wise cosine similarity patterns are very different from the layer-wise NN distances and ID patterns, and therefore cannot explain the latter. This insight is in line with the results from LLM-based analysis.
\vspace*{3mm}
\begin{figure*}[!htb]
    \centering
    \includegraphics[width=0.62\textwidth]{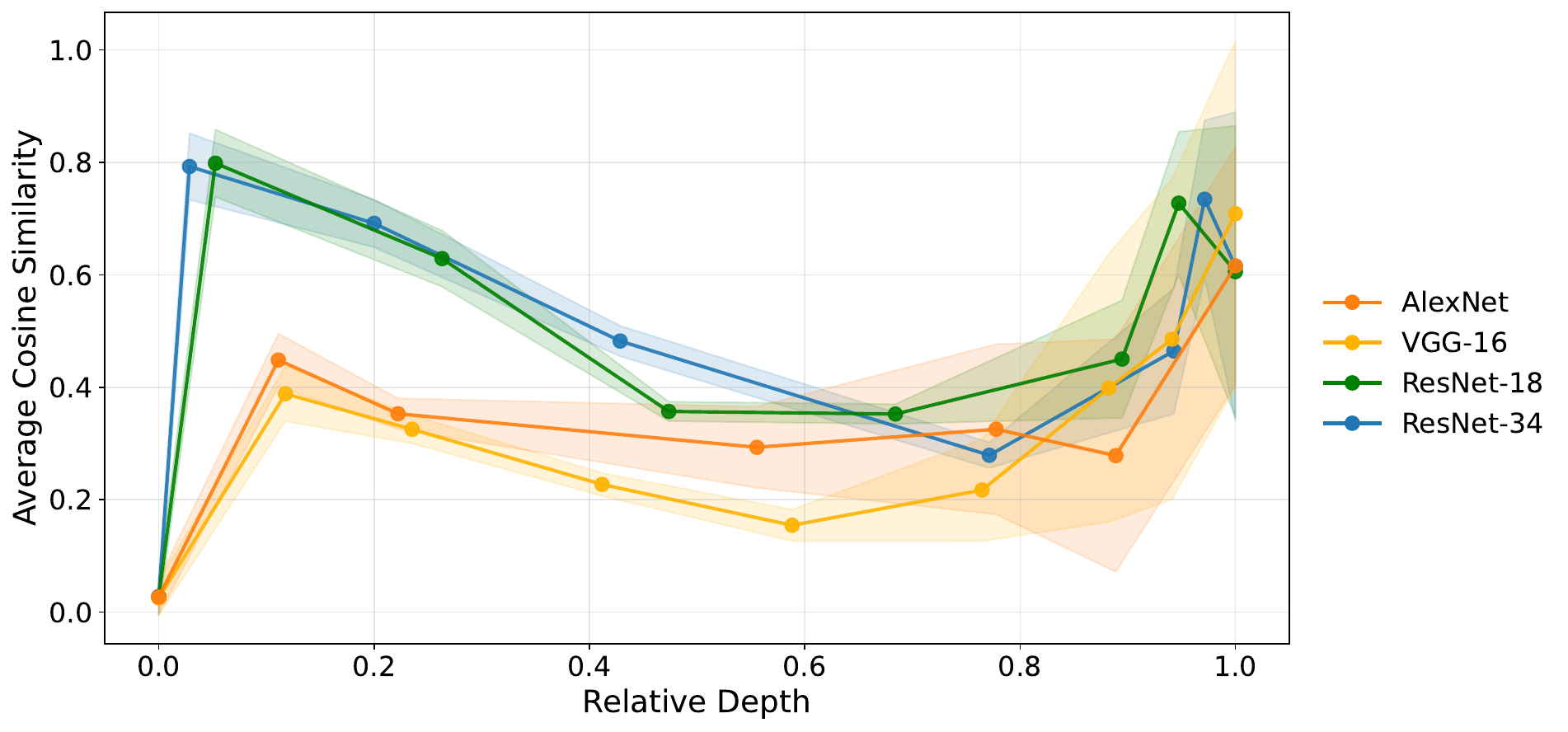}
    \vspace*{0mm}
    \caption{Average cosine similarity between layer-wise representations for different pre-trained convolutional architectures. The shaded area band represents twice the standard deviation.}
    \label{fig:avg_cos_sim_cnns}
\end{figure*}
\subsection{Representation Length Analysis}\label{app:rep_length}
\paragraph{LLMs} The results of the analysis of size-based analysis of layer representations from LLMs can be found in \cref{fig:nn_and_l2_llms}. The left plot depicts the average lengths of layer-wise representations from the three LLM models. The length of each representation is measured by the usual $L_2$ distance to the origin of the latent space. Given that these distances tend to increase over the hidden layers of all models (besides a constant dimension of the layer-wise latent space), this corresponds to an
expansion of representation in latent space. This phenomenon was also described in \cref{sec:size_of_reps}. A notable exception is the last layer. While the size of representations drastically increases in the case of the llama and mistral model, it drastically decreases for the pythia model. While these model-specific differences might be of interest on their own, they are of minor importance for our analysis. 

For our purposes, the key insight from \cref{fig:nn_and_l2_llms} is that the layer-wise $L_2$ lengths of representations (right plot) very closely match the layer-wise NN distances (left plot). Therefore, $L_2$ lengths of representations and the expansion in latent space as very likely to be the driver behind the observed NN distance. As the latter are used by ID estimators, they give rise to observed layer-wise ID patterns. 

The left plot in \cref{fig:nn_and_l2_llms} corresponds to an extension of \cref{fig:nn_dist_llms}. The former also includes the NN distances for last-layer representations, which are omitted in \cref{fig:nn_dist_llms}. As described above, the drastic change in NN distances in the last layer is induced by the drastic change in the length of representations in the last layer. As this phenomenon is not central to our discussion and makes it very hard to read of differences between NN distances, we omitted the last layer in \cref{fig:nn_dist_llms}.
\vspace{2mm}
\begin{figure*}[!htb]
    \centering
    \includegraphics[width=1\textwidth]{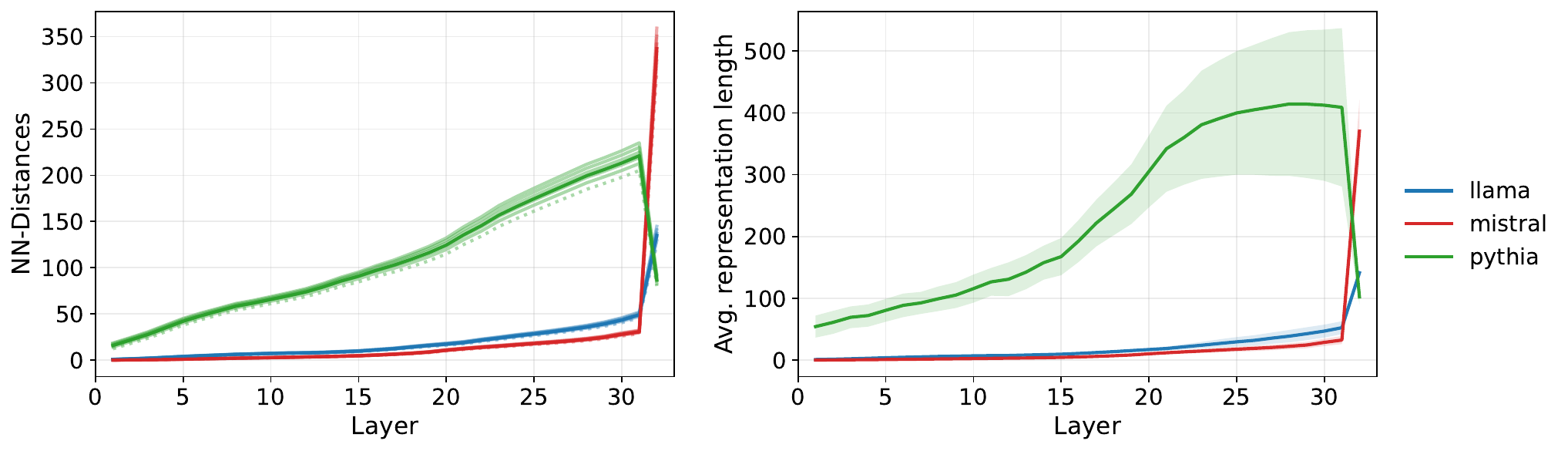}
    \caption{NN distances (left) and average length (measured by $L_2$-distances to the origin) of layer-wise representations for the llama, mistral, and pythia models. The left plot corresponds to \cref{fig:nn_dist_llms}, but additionally includes the last layer. The shaded area band in the right plot represents twice the standard deviation.}
    \label{fig:nn_and_l2_llms}
\end{figure*}
\paragraph{ViTs} The results for the length-based analysis for layer-wise representations of ViTs are depicted in \cref{fig:nn_and_l2_vits}. Analogous to the LLM-based results, the layer-wise patterns of $L_2$ length of representations generally grow over the hidden layers and closely resemble the layer-wise NN distance patterns.
\begin{figure*}[!htb]
    \centering
    \includegraphics[width=1\textwidth]{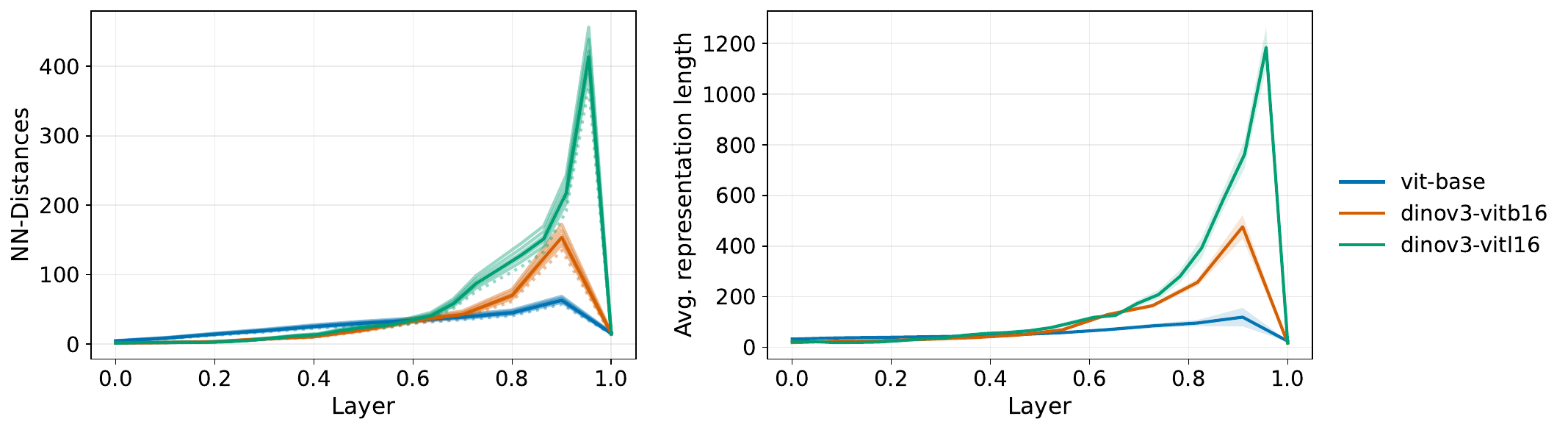}
    \caption{NN distances (left) and average length measured by $L_2$-distances to the origin (right) of layer-wise representations for different ViTs. In the left plot, for each model, each line (top to bottom) corresponds to the averages of $64^{th}$ \& $32^{nd}, \ldots, 2^{nd}$ \& $1^{st}$ NN distances. Solid lines denote the average over all 6 lines, with the average of the first two NN distances (used in the TwoNN estimator) highlighted as a dotted line. The shaded area band in the right plot represents twice the standard deviation.}
    \label{fig:nn_and_l2_vits}
\end{figure*}

\paragraph{CNNs} The results for the length-based analysis for layer-wise representations of the CNNs are depicted in \cref{fig:avg_length_reps_cnns}. Analogous to the other models, the layer-wise patterns of $L_2$ length of representations closely resemble the NN distance patterns found in \cref{fig:nn_dist_cnns} and are therefore likely the driving force behind the NN distances. However, in contrast to the LLM-based results, the lengths of representations do not exhibit an increasing but rather a decreasing trend over the layers, and therefore no expansion in latent space. It should be noted, however, that the dimension of the layer-wise ambient space of representations is not constant but mostly decreasing over the layers of CNNs (cf.~\cref{fig:id_vs_emb_dim}). Hence, a decrease in layer-wise representation lengths is expected for CNNs.

\begin{figure*}[!htb]
    \centering
    \includegraphics[width=1\textwidth]{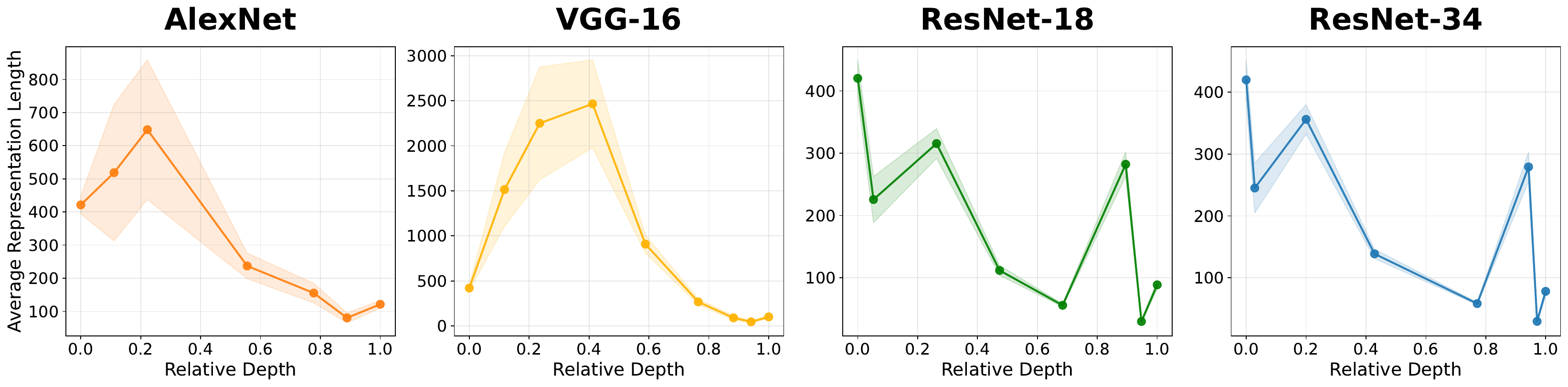}
    \caption{Average length (measured by $L_2$-distances to the origin) of layer-wise representations for different CNNs. The shaded area band represents twice the standard deviation.}
    \label{fig:avg_length_reps_cnns}
\end{figure*}
\subsection{Entropy vs.\ ID Analysis}\label{app:entropy_vs_id}
\paragraph{ViTs and CNNs} We also extended the analysis comparing layer-wise ID and entropy estimates in \cref{sec:other_metrics} from LLMs to ViTs and CNNs. The results for the ViT and CNN models are depicted in \cref{fig:id_vs_entropy_vits} and \cref{fig:id_vs_entropy}, respectively. Analogously to the findings obtained in the LLM-based analysis (cf.~\cref{fig:id_vs_entropy_llms}), we find a strong connection between the layer-wise patterns of estimated IDs and von Neumann entropy also for these models..
\begin{figure*}[!htb]
    \centering
    \includegraphics[width=0.65\textwidth]{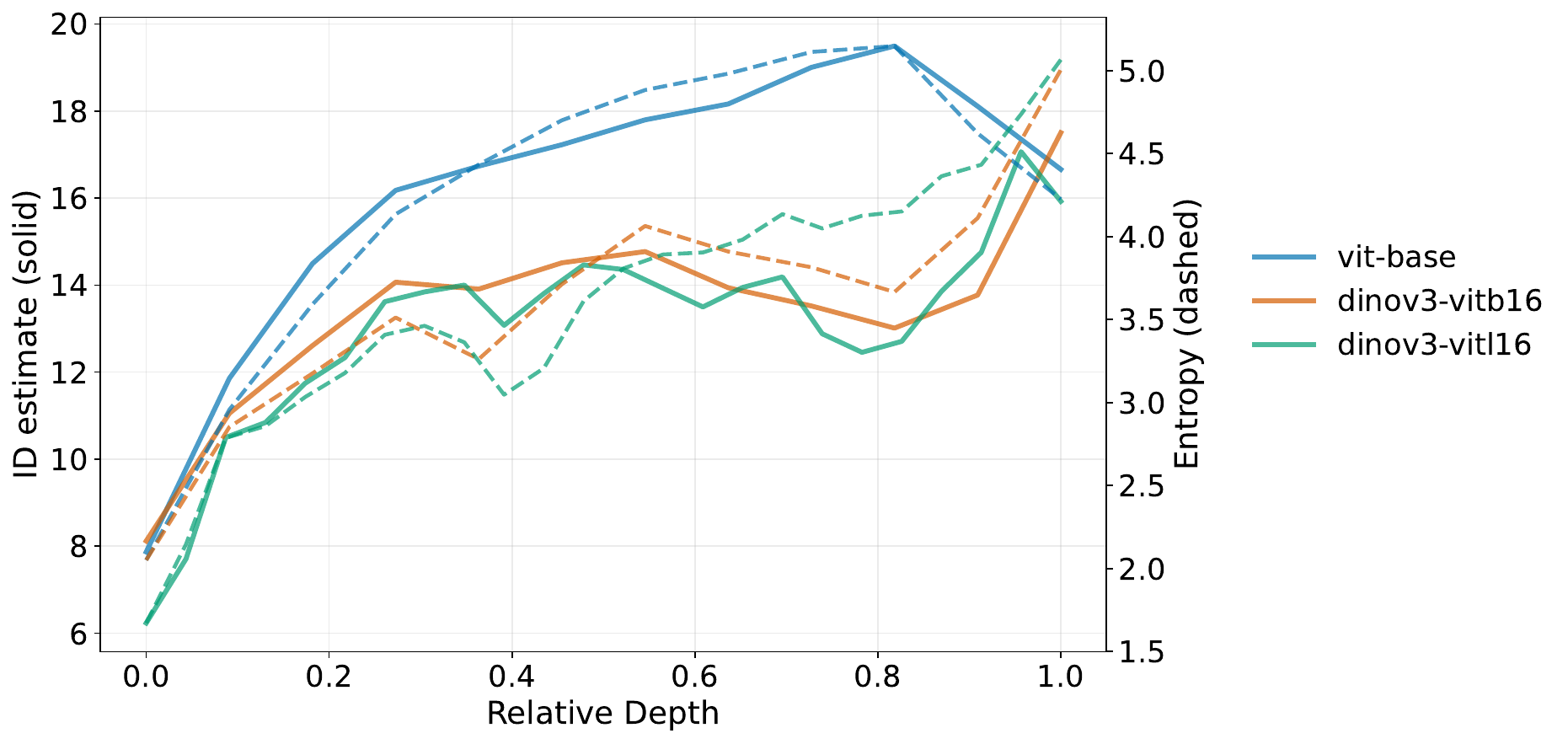}
    \caption{Layer-wise comparison of estimated intrinsic dimensions (left y-axis) vs.\ von Neumann entropy (right y-axis) of neural representations from different ViTs. 
    Details about the entropy metric can be found in \cref{app:metrics}.
    }
    \label{fig:id_vs_entropy_vits}
\end{figure*}
\begin{figure*}[!htb]
    \centering
    \includegraphics[width=1\textwidth]{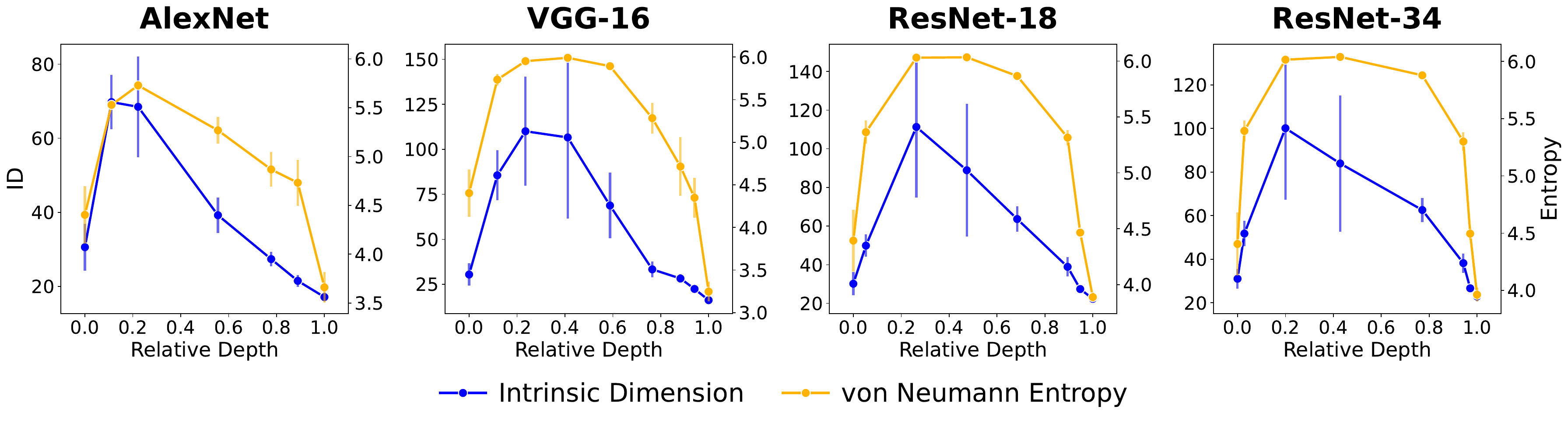}
    \caption{Layer-wise comparison of estimated intrinsic dimensions (left y-axis) vs.\ von Neumann entropy (right y-axis) of neural representations from different CNNs. 
    Details about the entropy metric can be found in \cref{app:metrics}.
    }
    \label{fig:id_vs_entropy}
\end{figure*}
\end{document}